\documentclass{article}
\usepackage[preprint]{neurips_2026}
\usepackage[utf8]{inputenc}
\usepackage[T1]{fontenc}
\usepackage{hyperref}
\usepackage{url}
\usepackage{microtype}
\hypersetup{hypertexnames=false}
\usepackage{amsmath}
\usepackage{amssymb}
\usepackage{amsthm}
\usepackage{booktabs}
\usepackage{multirow}
\usepackage{graphicx}
\graphicspath{{figures/}}
\usepackage{algorithm}
\usepackage{algorithmic}
\usepackage{enumitem}
\usepackage{xcolor}


\newtheorem{theorem}{Theorem}



\newcommand{\E}{\mathbb{E}}
\newcommand{\Var}{\mathrm{Var}}
\newcommand{\Cov}{\mathrm{Cov}}
\renewcommand{\Pr}{\mathrm{Pr}}
\newcommand{\calR}{\mathcal{R}}
\newcommand{\calI}{\mathcal{I}}
\newcommand{\calC}{\mathcal{C}}

\newcommand{\revised}[1]{#1}

\title{LLM Judge Validation Under Sparse Overlap: \\ From Inference to Design}

\author{%
  Junxuan Li \\
  Adobe\\
  \texttt{junxuanl@adobe.com} \\
  \And
  Arko Mukherjee \\
  Adobe\\
  \texttt{arkom@adobe.com} \\
  \And
  Soumyabrata Pal \\
  Adobe Research\\
  \texttt{soumyabratap@adobe.com} \\
}

\begin{document}
\maketitle

\begin{abstract}
Validating an LLM-as-a-judge requires estimating its agreement with humans, yet annotation budgets rarely allow every item to be multiply labeled. We prove that this \emph{overlap sparsity} is the first-order determinant of wrong deployment decisions: at 5\% pairwise overlap, wrong-decision rates reach 25\% and the probability of selecting the wrong best judge among ten candidates is 65\%. The two actionable levers are overlap \emph{quantity} and \emph{allocation}. For quantity, we derive a minimum-overlap formula showing $\rho \geq 0.25$ suffices for non-borderline judges while borderline cases remain fundamentally hard. For allocation, a zero-cost stratified scheme halves false-rejection rates relative to random sampling when strata are informative. We validate on 10 LLM judges across four evaluation matrices spanning visual assessment, causal reasoning, and summarization.
\end{abstract}

\section{Introduction}
\label{sec:intro}

Large Language Models (LLM) and Vision Language Models (VLM) are increasingly used as automated evaluators (e.g. ``LLM-as-a-judge'') for assessing text generation quality \citep{zheng2023judging}, safety compliance, and instruction following.
Before deploying such a judge, practitioners must validate: \emph{does it agree with human annotators well enough to replace them?}
The standard approach is to compute inter-annotator agreement metrics between the judge and human annotators on a calibration set, and apply a statistical test to determine whether the agreement is sufficient.

This approach faces a fundamental \textbf{annotation allocation problem}: teams can often afford broad single-pass coverage, but only a small subset of items can receive multiple labels. The practical question is \emph{where} to spend the next limited batch of annotations and \emph{how} to assess judge quality given the resulting annotation plan.
Figure~\ref{fig:depth_breadth} illustrates: under the same budget, concentrating labels on a narrow slice inflates the apparent human ceiling, while spreading them too thinly leaves too little shared overlap; only a stratified design avoids both failures (setup in Appendix~\ref{app:depth_breadth}).

\begin{figure*}[h]
\centering
\includegraphics[width=0.82\textwidth]{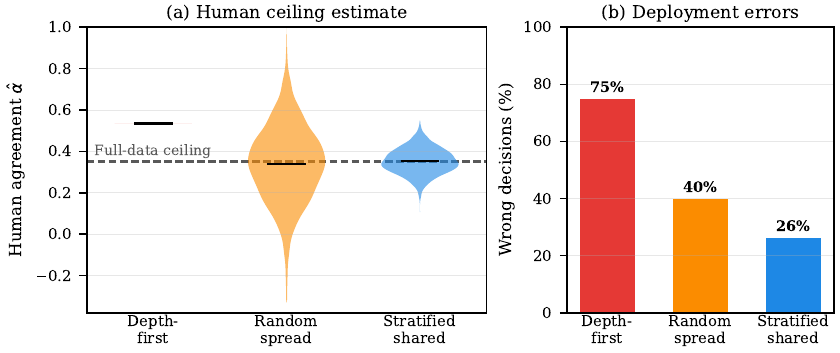}
\caption{\textbf{The same annotation budget can fail in two different ways.} On a real visual-assessment task, depth-first labeling biases the human-agreement ceiling~(a) and inflates wrong deployment decisions~(b); random spread gives each annotator independent samples, leaving too little shared overlap for reliable estimation; \revised{\textsc{Stratified} shared overlap draws a proportional stratum sample that all annotators label, simultaneously fixing fragmentation (maximal pairwise co-coverage) and representativeness (stratum-proportional distribution).}}
\label{fig:depth_breadth}
\end{figure*}

This motivates our systematic study of agreement estimation under sparse overlap.
Existing validation methods \citep{calderon2025alttest, jung2025trust, he2025hide} assume dense calibration data and provide little guidance on how reliability degrades when overlap is scarce.
We characterise how sparsity affects estimator bias, variance, and wrong-decision rates (\S\ref{sec:problem}), and show that overlap quantity and allocation are the dominant levers for reducing wrong decisions, whereas the choice of agreement coefficient plays a secondary role (\S\ref{sec:sample_complexity},~\S\ref{sec:method}). We derive concrete overlap guidelines (\S\ref{sec:sample_complexity},~\S\ref{sec:rq2}) and validate them on 10 LLM judges across three benchmark sources, yielding four evaluation matrices (\S\ref{sec:experiments}).

\textbf{Contributions.} \quad \textbf{C1.\ Analytical characterisation of sparsity-induced wrong decisions.} We prove that sparse agreement estimator variance decomposes as $\gamma_F(\pi)/m$ (Thm.~\ref{thm:sparsity}), where the coefficient-specific amplification $\gamma_F$ is bounded for $\hat p_o$ and AC1 but diverges under label skew for $\hat\kappa$ and $\hat\alpha$, and show empirically that \revised{among the design levers practitioners control at validation time ($F$, $\rho$, $\mathbf{D}$), overlap is the necessary condition for reliable decisions, while coefficient and design choices provide secondary improvements within each overlap level} (\S\ref{sec:problem},\,\S\ref{sec:rq1}).\quad \textbf{C2.\ Annotation design as a first-class methodological lever.} We establish that zero-cost static stratification is a strong default on real design-quality diagnostics when the strata are informative, while sequential coverage can help under extreme skew (\S\ref{sec:method},\,\S\ref{sec:rq2}).\quad \textbf{C3.\ Overlap planning.} An asymptotic minimum-overlap recipe for certification (Thm.~\ref{thm:planning}(ii)) and ranking (Thm.~\ref{thm:planning}(iii)), plus empirical reliability tables, let practitioners reason about ``how much overlap is enough?'' before or during data collection, validated on 10 LLM judges across three benchmark sources (\S\ref{sec:experiments}).

\textbf{Relation to prior work.}
\textit{LLM-as-a-judge validation.} A growing line of work validates LLM judges \citep{zheng2023judging, gu2026survey, li2024llmsjudges} via leave-one-out testing \citep{calderon2025alttest}, conformal prediction \citep{jung2025trust}, TOST equivalence \citep{he2025hide}, reporting-protocol audits \citep{lee2025correctly}, and indeterminacy-aware analyses \citep{guerdan2025indeterminacy}. These pipelines assume dense calibration data and do not analyse how their decision rules behave when each pair of annotators co-labels only a small slice of the corpus. \revised{Our problem structure is asymmetric and distinct from crowdsourcing: the LLM judge has already labelled all $n$ items at near-zero marginal cost, and the scarce resource is the human overlap needed to validate it. The question is not how to allocate labels to improve label quality (crowdsourcing) but how much human overlap suffices to make a reliable deployment decision.}

\textit{Inter-annotator agreement under sparsity.} The textbook coefficients ($\kappa$, $\alpha$, AC1; \citealp{cohen1960coefficient, fleiss1971measuring, krippendorff2004content, gwet2008computing, artstein2008inter}) and their finite-sample / sample-size diagnostics \citep{sim2005kappa, hughes2024improved, james2026counting, apple2024efficient, feinstein1990high, byrt1993bias} are derived for fully observed calibration sets. \revised{\citet{norregaard2022spa} show that agreement estimated from an already-collected sparse matrix is noisy; we build on this observation by asking the prospective question: given that sparse estimation is unreliable, how should the annotation budget be allocated \emph{before} data collection to make the resulting deployment decision reliable?} We characterise how each coefficient's bias and variance scale with the per-pair overlap count and translate that into overlap-rate guidelines for the sparse regime.

\textit{Imputation-based annotation modelling.} Latent-truth and incomplete-block models \citep{dawid1979maximum, hovy2013learning, paun2018comparing, passonneau2014benefits, snow2008cheap, fleiss2003statistical} \emph{impute} missing labels from generative assumptions about raters. We ask the orthogonal question of how much direct pairwise overlap is enough for unmodelled estimation, sidestepping the generative-model dependence.

\textit{Annotation design and disagreement-as-signal.} Allocating a fixed annotation budget across raters is treated operationally rather than as a methodological lever in current LLM-judge practice; our overlap guidelines also apply when disagreement is treated as signal rather than noise \citep{plank2022problem}\revised{, or when the target is defined by a social-choice aggregation rule rather than majority vote \citep{endriss2013collective, conitzer2024social}}. \revised{Our work addresses the measurement question---can we measure agreement reliably under sparsity?---and is complementary to the choice-of-target question (what should we measure?), which is orthogonal and treated in the cited social-choice literature.}

\revised{\textit{Related estimation methods.} \citet{fisch2024stratppi} (StratPPI) and \citet{angelopoulos2023ppi} (PPI) use model predictions as control variates to tighten confidence intervals for population performance metrics (e.g.\ mean quality score). Both are post-hoc estimators applied after data collection; neither addresses inter-annotator agreement estimation, annotation allocation, or judge accept/reject decisions. Appendix~\ref{app:prior_work_comparison} provides a structured comparison.}

\section{Agreement Validation Under Sparsity}
\label{sec:problem}

\subsection{Notation}
\label{sec:notation}

Let $\calI = \{1, \ldots, n\}$ denote the evaluation items and $\calR = \{r_1, \ldots, r_K, j\}$ the $K$ human annotators plus judge $j$.
Each rater assigns labels from $\calC = \{c_1, \ldots, c_L\}$ to the items they see.
Let $\mathbf{A} \in (\calC \cup \{\texttt{NA}\})^{n \times (K+1)}$ denote the annotation matrix and $\mathbf{D} \in \{0,1\}^{n \times (K+1)}$ the corresponding observation mask, with $D_{ik} = \mathbf{1}[A_{ik} \neq \texttt{NA}]$.
We treat $\mathbf{A}$ and $\mathbf{D}$ as random: the item-level annotation vectors $(A_{i,1}, \ldots, A_{i,K+1})$ are i.i.d.\ draws from an unknown joint label distribution $P$ over $\calC^{K+1}$, and the entries of $\mathbf{D}$ follow the design distribution specified by the annotation protocol.
Throughout, bare symbols denote \emph{population} quantities (deterministic functionals of $P$), and hats denote \emph{estimators}.
Write $\pi = (\pi_1, \ldots, \pi_L)$ for the population label prevalence and $\pi_{\max} = \max_\ell \pi_\ell$.
Throughout, \emph{uniform prevalence} means roughly balanced label frequencies, whereas \emph{skewed prevalence} means one or a few labels dominate the corpus (large $\pi_{\max}$).
The pairwise overlap rate for raters $(k, k')$ is $\rho_{kk'} = \tfrac{1}{n} \sum_i D_{ik} D_{ik'}$, and the global overlap rate is $\rho = \tfrac{1}{\binom{K+1}{2}} \sum_{k < k'} \rho_{kk'}$.

In the \emph{dense} setting ($\rho = 1$), every rater annotates every item.
In the \emph{sparse} setting that motivates this work, the LLM judge annotates all $n$ items ($D_{ij} = 1$ for all $i$), and for exposition we write the corpus-wide first-pass label layer as if one primary human annotator annotates all items ($D_{ir_1} = 1$ for all $i$), while each additional human annotator $r_k$ independently annotates each item with probability $\rho_k$, i.e.\ $D_{ir_k} \sim \mathrm{Bernoulli}(\rho_k)$ for $k = 2, \ldots, K$.
The first-pass layer can also be a pooled pass across multiple annotators or an auxiliary AI source; without one, this becomes a general incomplete-block allocation problem \citep{fleiss2003statistical} with the same sparsity bottleneck.

\subsection{Agreement Coefficients}
\label{sec:agreement_coefficients}

\textbf{Pairwise observed agreement.}
For any pair of raters $(k,k')$, the overlap set is $\calI_{kk'} = \{i : D_{ik} = D_{ik'} = 1\}$ and the pairwise observed agreement is
\begin{equation}
\label{eq:po}
\hat p_o^{(kk')} \;=\; \frac{1}{|\calI_{kk'}|}\sum_{i \in \calI_{kk'}} \mathbf{1}[A_{ik}=A_{ik'}].
\end{equation}
$\hat p_o$ is the default scoring function in our pipeline (\S\ref{sec:pipeline}). Under sparse overlap $|\calI_{kk'}|$ can be very small, making $\hat p_o$ noisy.

\textbf{Chance-corrected coefficients.}
We also adopt three chance-corrected agreement coefficients as alternative scoring functions (see full forms in Appendix~\ref{app:coef_forms}):
\begin{itemize}

\item \textbf{Krippendorff's $\hat\alpha$} \citep{krippendorff2004content} pools all within-item pairs into a ratio $\hat\alpha = 1 - \hat D_o/\hat D_e$ of observed-vs-expected disagreement.

\item \textbf{Cohen's pairwise $\hat\kappa$} \citep{cohen1960coefficient} is the analogous $(\hat p_o^{(kk')}-\hat p_e^{(kk')})/(1-\hat p_e^{(kk')})$ ratio with rater-marginal chance correction $\hat p_e^{(kk')} = \sum_\ell \hat\pi_\ell^{(k)}\hat\pi_\ell^{(k')}$; it suffers the \textbf{prevalence paradox} \citep{feinstein1990high,byrt1993bias} as $\pi_{\max} \to 1$, as $1-p_e$ vanishes.

\item \textbf{Gwet's $\widehat{\mathrm{AC1}}$} \citep{gwet2008computing} swaps $\hat p_e$ for a marginal-symmetric chance term $\frac{1}{L-1}\sum_\ell \hat\pi_\ell(1-\hat\pi_\ell)$, yielding bounded denominator constants at the cost of a different interpretation.
\end{itemize}

\textbf{Sparse-form estimators.}
All four coefficients are evaluated on the same sparse overlap subsets in \S\ref{sec:pipeline}. We write $m$ for the number of items in the overlap subset on which an estimator is evaluated, and $m_k$ when this count is specific to held-out rater $k$; estimators are subscripted accordingly: $\hat p_{o,m}$, $\hat\alpha_m$, $\hat\kappa_m$, $\widehat{\mathrm{AC1}}_m$. For $\hat p_{o,m}$ this is the itemwise average~\eqref{eq:po} restricted to the overlap subset; for the ratio statistics it is the subset-level analogue with sums and marginals over only the multi-annotated items (full forms in Appendix~\ref{app:coef_forms}).

\begin{theorem}[Sparsity--prevalence decomposition of agreement scores]
\label{thm:sparsity}
Assume i.i.d.\ items with $L \geq 2$ label categories, an overlap mask independent of labels, overlap count $m$, and fixed rater marginals (exchangeable in the simplified derivations below). Then for each coefficient $F \in \{\hat p_o, \hat\alpha, \hat\kappa, \widehat{\mathrm{AC1}}\}$, the sparse estimator satisfies
\begin{equation}\label{eq:var_decomp}
  \Var(\hat F_{m}) \;=\; \frac{\gamma_F(\pi)}{m} \;+\; O(1/m^2),
\end{equation}
where the coefficient-specific amplification factor $\gamma_F$ exhibits a prevalence-divergence hierarchy:
\begin{enumerate}[label=\textup{(\roman*)},leftmargin=*]
\item \textbf{$\hat p_o$}: $\gamma_{p_o} = p_o(1{-}p_o)$ \textup{(exact; no chance-correction denominator)}.
\item \textbf{$\hat\alpha$}: $\gamma_\alpha = O\bigl((1-p_e^\alpha)^{-2}\bigr)$, diverging as $\pi_{\max} \to 1$.
\item \textbf{$\hat\kappa$}: $\gamma_\kappa = O\bigl((1-p_e^\kappa)^{-4}\bigr)$, \textbf{diverging faster} than $\gamma_\alpha$ (exponent $-4$ vs.\ $-2$).
\item \textbf{$\widehat{\mathrm{AC1}}$}: $\gamma_{\mathrm{AC1}}$ bounded for all $L,\pi$ because $1-p_e^{\mathrm{AC1}} \geq (L{-}1)/L \geq 1/2$ \textup{(denominator cannot vanish)}.
\end{enumerate}
\textup{Bias: $\E[\hat F_{m}] - F = O(1/m)$ for all four coefficients, with $\hat p_o$ exactly unbiased (proof in Appendix~\ref{app:proof_sparsity}).}
\end{theorem}

Under uniform prevalence the amplification constants $\gamma_F$ are comparable across coefficients; under skew the divergence of $\gamma_\kappa$ and $\gamma_\alpha$ becomes the dominant source of wrong decisions (\S\ref{sec:rq1}).

\subsection{Validation Pipeline}
\label{sec:pipeline}

Agreement-based validation asks whether judge $j$ agrees with humans \emph{at least as well as} humans agree with each other \citep{calderon2025alttest, jung2025trust, he2025hide}. We adopt the standard leave-one-out pipeline from these prior works: each human rater $r_k$ ($k = 1,\ldots,K$) is held out in turn, and the remaining $K-1$ humans $r_{-k} = \{r_{k'} : k' \neq k\}$ form the \emph{held-in pool}.

The main pipeline uses $F=\hat p_o$. Given a held-out rater $r_k$, the per-rater overlap subset is $\calI_k = \{i : D_{ir_k} = D_{ij} = 1,\; \sum_{k' \neq k} D_{ir_{k'}} \geq 1\}$, with per-rater overlap count $m_k = |\calI_k|$ ($\E[m_k] = \rho_k n$); under the sparse design of \S\ref{sec:notation} the bottleneck is $\min_{k \geq 2} m_k$ since $r_1$ is densely labelled. On this subset, the \emph{judge--pool} and \emph{human--pool} scores average pairwise observed agreement over within-item pairs against the held-in raters $r_{-k}$:
\begin{align}
\label{eq:score_judge}
S(j;\, r_k) \;&=\; \frac{\sum_{k' \neq k}\sum_i D_{ir_k} D_{ir_{k'}} D_{ij}\,\mathbf{1}[A_{ij}=A_{ir_{k'}}]}{\sum_{k' \neq k}\sum_i D_{ir_k} D_{ir_{k'}} D_{ij}}, \\
\label{eq:score_human}
S(r_{-k};\, r_k) \;&=\; \frac{\sum_{k' \neq k}\sum_i D_{ir_k} D_{ir_{k'}}\,\mathbf{1}[A_{ir_k}=A_{ir_{k'}}]}{\sum_{k' \neq k}\sum_i D_{ir_k} D_{ir_{k'}}}.
\end{align}
When RQ1 swaps in $\hat\alpha$, $\hat\kappa$, or $\widehat{\mathrm{AC1}}$, the same leave-one-out subsets are reused with the corresponding subset-level ratio statistic from Appendix~\ref{app:coef_forms}. The per-rater outcome records whether the judge matches or exceeds the human ceiling up to a tolerance $\varepsilon \geq 0$, and the judge is accepted when at least half of the $K$ comparisons favour it:
\begin{align}
\label{eq:per_rater}
    T_k &\;=\; \mathbf{1}\!\bigl[\,S(j;\,r_k) \;\geq\; S(r_{-k};\,r_k) - \varepsilon\,\bigr], \\
\label{eq:winning_rate}
    \omega &\;=\; \frac{1}{K} \sum_{k=1}^{K} T_k, \qquad \text{accept judge iff } \omega \geq 0.5.
\end{align}
The choice of $\varepsilon$ reflects annotators' subject-matter expertise \citep{calderon2025alttest}; we fix $\varepsilon=0.05$ throughout. By Thm.~\ref{thm:sparsity}, $S(j;r_k) - S(r_{-k};r_k)$ carries $O(1/m_k)$ sampling variance (with chance-corrected ratios inheriting the same rate by the delta method), so at low $\rho$ noise can flip the per-rater comparison in either direction, producing false approvals on weak judges and false rejections on strong ones regardless of the agreement coefficient.

\textbf{Pipeline-generality.} The LOO rule (Eqs.~\eqref{eq:per_rater}--\eqref{eq:winning_rate}) is the de facto LLM-judge validation protocol \citep{calderon2025alttest, jung2025trust, he2025hide}. Three of our four results are pipeline-general: the $\gamma_F/m$ decomposition (Thm.~\ref{thm:sparsity}), the \textsc{Strat} bias-reduction mechanism (\S\ref{sec:stratified}), and the planning recipe (Thm.~\ref{thm:planning}). What is LOO-specific is the exact wrong-decision threshold; alternative pipelines (TOST \citep{he2025hide}, conformal prediction \citep{jung2025trust}) face the same $\gamma_F/m$ noise and the planning recipe applies with $\delta$ set to each pipeline's tolerance.

\subsection{Overlap Planning}
\label{sec:sample_complexity}

A natural question is how much overlap is enough. Under a CLT approximation, the answer reduces to a Gaussian-tail recipe
\begin{equation}
\label{eq:overlap_recipe}
m^{*} \;=\; \Bigl\lceil z_{1-\alpha_{\text{sig}}/2}^{\,2}\,\sigma^2 \,\big/\, \delta^2 \Bigr\rceil,
\end{equation}
where $\sigma^2$ is the per-item variance and $\delta$ is a tolerance the practitioner is willing to miss with probability at most $\alpha_{\text{sig}}$. Theorem~\ref{thm:planning} instantiates this recipe for two core tasks---certifying a single judge and ranking multiple candidates---connecting the required overlap to the amplification hierarchy of Thm.~\ref{thm:sparsity}. \S\ref{sec:rq2} fixes $\delta = \varepsilon = 0.05$ as a real-benchmark diagnostic and \S\ref{sec:rq4} tabulates minimum $\rho$ across $\delta$ levels. \emph{Allocation} of this budget is the subject of \S\ref{sec:method}.

\begin{theorem}[Overlap planning under sparsity]
\label{thm:planning}
Under the conditions of Theorem~\ref{thm:sparsity}, the per-rater score difference $d_k = S(j;r_k) - S(r_{-k};r_k)$ has mean $\mu_d := \E[d_k]$ and variance $\sigma_0^2/m$, where $\sigma_0^2$ inherits the amplification hierarchy of Thm.~\ref{thm:sparsity}: bounded for $\hat p_o$ and $\widehat{\mathrm{AC1}}$, prevalence-dependent for $\hat\kappa$ and $\hat\alpha$. Assume signal margin $\mu_d + \varepsilon > 0$. Then under the Gaussian approximation:
\begin{enumerate}[label=\textup{(\roman*)},leftmargin=*]
\item \textup{(False-rejection rate).} $\Pr[\,d_k < -\varepsilon\,] \approx \Phi\!\bigl(-(\mu_d + \varepsilon)\sqrt{m}\,/\,\sigma_0\bigr)$.
\item \textup{(Certification).} For per-rater false-rejection rate $\leq \alpha_{\textup{sig}}/2$: $m \geq m^{*}_{\textup{cert}} = \lceil z_{1-\alpha_{\textup{sig}}/2}^{\,2}\,\sigma_0^2 / (\mu_d+\varepsilon)^2\rceil$. Since $\sigma_0^2$ inherits $\gamma_F(\pi)$, required overlap is prevalence-sensitive for $\hat\kappa$ and $\hat\alpha$.
\item \textup{(Ranking).} Selecting the best among $J$ judges replaces $(\mu_d{+}\varepsilon)^2$ with the minimum pairwise margin $\Delta_{\min}^2$, inflates the quantile to $z_{1-\alpha/(J-1)}$ (Bonferroni), and doubles the variance; these effects compound, so ranking demands more overlap than certification.
\end{enumerate}
\textup{(Proof in Appendix~\ref{app:proof_planning}.)}
\end{theorem}

\section{Stratified Annotation Design}
\label{sec:method}
\label{sec:stratified}

Given the pipeline of \S\ref{sec:pipeline}, the practitioner's remaining choices are the agreement coefficient $F$, the overlap budget $\rho$, the rater count $K$, and the annotation design $\mathbf{D}$. Each acts on a different term of the $\gamma_F(\pi)/m$ decomposition~\eqref{eq:var_decomp}: $F$ determines the amplification $\gamma_F$, $\rho$ controls $m$ directly, and $\mathbf{D}$ controls the label composition of the overlap subset. This section develops the design lever; \S\ref{sec:rq1} confirms their relative importance empirically.

\textbf{Marginal-anchored stratification.}
We instantiate the sparse design of \S\ref{sec:notation} with a per-rater overlap subset structure: the judge and one primary human annotate all $n$ items, and each additional human $r_k$ labels an $m$-item subset $\calI^+_k \subseteq \calI$. Under purely random selection, the empirical label distribution on $\calI^+_k$ can diverge from the corpus marginals $\pi$ when prevalence is skewed. This marginal mismatch feeds into the expected-disagreement term $\hat D_e$, biasing the human ceiling estimate that the validation pipeline compares the judge against. \textbf{Stratified subsampling} draws $\calI^+_k$ proportionally across pre-specified strata, reducing this mismatch in expectation when the stratification signal tracks the relevant label marginals. The design can reduce marginal-mismatch bias without sacrificing per-rater overlap count $m_k$, but it is not guaranteed to help when the stratum signal is uninformative or collapses under extreme skew.

\textbf{A Family of Stratification Designs.}
A stratified design is fully specified by (i) a per-rater \emph{stratification signal} $s_k : \calI \to \calC$ that assigns every item a stratum label, and (ii) a within-stratum sampler that draws $m_\ell \propto |\{i : s_k(i) = c_\ell\}|$ items per stratum (proportional allocation, residual slots distributed to the largest strata). We consider three concrete members of this family along a complexity spectrum:
\begin{itemize}[leftmargin=*, itemsep=2pt, topsep=2pt]
    \item \textsc{Strat}: $s_k(i) = A_{i, r_1}$, the primary annotator's label, shared across all secondaries. Fully static, the simplest member.
    \item \textsc{Seq-Refined}: $s_k(i) = \mathrm{mode}\{A_{i, r_j} : j < k,\, D_{i, r_j} = 1\}$, the per-item consensus among raters whose labels are already in hand when rater $k$'s task is configured.
    \item \textsc{Seq-Coverage}: \textsc{Seq-Refined} with within-stratum ties broken in favour of items least covered so far, equalising pairwise overlap across raters.
\end{itemize}
The two sequential members require only \emph{between-rater} adaptation: rater $k$'s task is fixed before their session begins using labels already collected, so they are operationally feasible whenever annotation jobs are queued sequentially. We exclude \emph{within-rater} per-item adaptive policies on operational grounds, since annotation sessions are often batched, time-bounded, and run asynchronously across annotators. \S\ref{sec:rq2} benchmarks the stratification design family.

\textbf{Fragmentation and representativeness.} These are distinct failure modes. \emph{Fragmentation}: with $K$ raters each covering $m$ items independently, expected pairwise co-coverage is only $m^2/n$, collapsing well before per-rater coverage is useful. \emph{Representativeness}: even a shared panel can under-sample rare strata under skew, biasing the human ceiling. \textsc{Strat} addresses both: stratified draws control the label distribution while the shared panel structure maximises co-coverage.

\section{Experiments}
\label{sec:experiments}

\subsection{Setup}
\label{sec:setup}

\textbf{Synthetic data.} We generate annotation matrices with $n=500$ items, $K \in \{4,5\}$ humans plus one LLM judge, and $L \in \{2,5\}$ labels (Appendix~\ref{app:synthetic}), varying judge accuracy $p_j$, human ceiling accuracy $p_h$ (fixed at $0.85$), prevalence skew $\pi_{\max}$, and overlap rate $\rho$ over task-specific grids and averaging over 300 realisations.

\textbf{Real benchmarks and judges.} We use three benchmark sources (\textbf{WAX} \citep{codella2018skin}, \textbf{CeBaB} \citep{abraham2022cebab} with \texttt{stars}/\texttt{aspects} subtasks, \textbf{SummEval} \citep{fabbri2021summeval}), yielding four real evaluation matrices, with \textbf{10 LLM judges}: five from \citet{calderon2025alttest} (Gemini Flash/Pro, GPT-4o-mini, Llama-3.1, Mistral-v0.3) plus five collected for this work (GPT-5.4, GPT-5.2, GPT-4o, Claude~Opus~4.5, Claude~Sonnet~4.5; 28{,}822 new judge labels). For \S\ref{sec:rq2} we apply a \textbf{class-merging skew-masking protocol} that binarises each multi-class matrix at threshold $t$ to expose a tunable $\pi_{\max} \in [0.55, 0.96]$, since no benchmark in our suite naturally reaches the $\pi_{\max} \geq 0.90$ regime the sequential family members target.

\textbf{Inference convention and reference.}
All results in \S\S\ref{sec:rq1}--\ref{sec:rq4} use the leave-one-out pipeline (Eqs.~\eqref{eq:per_rater}--\eqref{eq:winning_rate}, $\varepsilon = 0.05$) with $F = \hat p_o$ unless noted; the coefficient swap to $F \in \{\hat\alpha,\hat\kappa,\widehat{\mathrm{AC1}}\}$ is RQ1's comparison. Decisions on the dense matrix ($\rho=1$) serve as the finite-sample reference scored by wrong-decision rates.

\subsection{RQ1: What Drives Wrong Decisions -- Coefficient or Design?}
\label{sec:rq1}

\begin{figure*}[!thbp]
    \centering
    \includegraphics[width=0.92\textwidth]{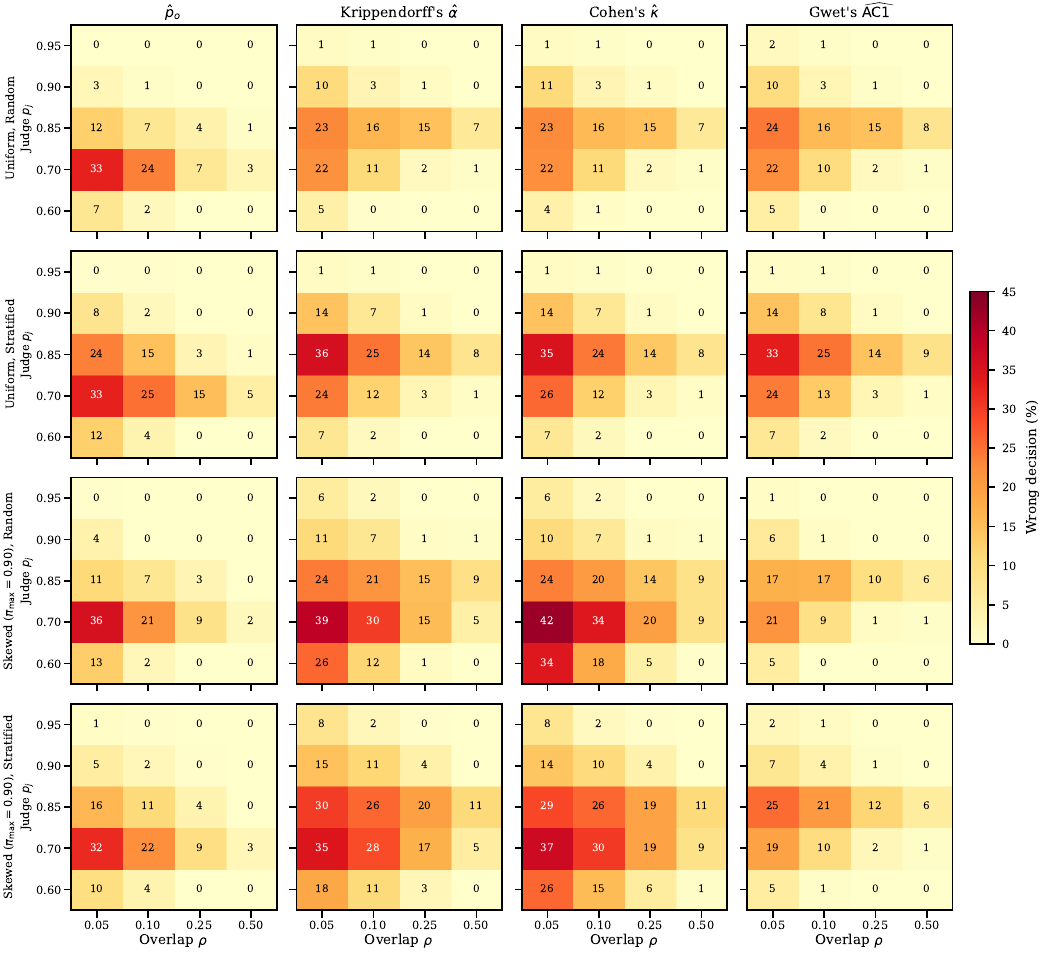}
    \caption{\textbf{Wrong-decision rate across overlap, judge quality, and design.} Each cell shows the percentage of trials whose deploy/reject outcome differs from the ground-truth reference. Row pairs vary prevalence and annotation design (Random vs.\ Stratified); columns vary the agreement coefficient $F$.}
    \label{fig:wrong_heatmap}
    \end{figure*}

\S\ref{sec:method} showed that $F$, $\rho$, and $\mathbf{D}$ each act on a different term of the $\gamma_F/m$ decomposition~\eqref{eq:var_decomp}; RQ1 isolates their effects on wrong-decision rates. Borderline judges ($p_j \in \{0.85, 0.70\}$) are the diagnostic regime: small signal margins make the $O(1/m_k)$ noise most likely to flip the per-rater comparison.

\textbf{$\hat p_o$ is the most stable default in the diagnostic cells.} In the borderline strong-judge cells of Figure~\ref{fig:wrong_heatmap}, $\hat p_o$ has the lowest or near-lowest wrong-decision rate, with the gap largest at low overlap: at $\rho = 0.05$, $p_j = 0.90$, $\hat p_o$ errs on 3--4\% of trials while $\hat\alpha$ and $\hat\kappa$ reach 10--11\% from the diverging amplification $\gamma_\alpha$, $\gamma_\kappa$ (Thm.~\ref{thm:sparsity}(ii--iii)). Large-gap judges ($p_j = 0.95$) are robust to coefficient choice as the margin absorbs the noise.

\textbf{Stratification helps under skew but not uniformly.} Comparing Random and Stratified row-pairs of Figure~\ref{fig:wrong_heatmap}, stratification reduces wrong decisions for weak judges under skew (e.g.\ $p_j = 0.70$, $\hat\kappa$, $\rho = 0.05$: $42\% \to 37\%$) but offers no consistent benefit under uniform prevalence, where marginal-mismatch bias is absent. The full design effect emerges on real benchmarks with natural skew (\S\ref{sec:rq3}).

\subsection{RQ2: Which Design Performs Best on Real Benchmarks?}
\label{sec:rq2}

RQ2 compares the four design members of \S\ref{sec:stratified} (\textsc{Random}, \textsc{Strat}, \textsc{Seq-Refined}, \textsc{Seq-Coverage}) on real benchmarks using test-independent metrics that evaluate how faithfully the human-pool score is preserved under sparse overlap. For this diagnostic we use Krippendorff's $\hat\alpha_{\mathrm{hh}}$, which is sensitive to the marginal-mismatch mechanism stratification controls. Write $\hat F_m$ for the sparse human-pool estimate and $F^{*}$ for its dense-matrix reference ($\rho=1$):
\begin{itemize}[leftmargin=*, itemsep=2pt, topsep=2pt]
    \item \textbf{Bias:} $\E[\hat F_m] - F^{*}$, the expected drift of the sparse estimate. This directly measures how the sampled overlap subset changes the human-ceiling estimate relative to the dense matrix.
    \item \textbf{Std:} $\mathrm{std}(\hat F_m)$, sampling spread across trials.
    \item \textbf{$(\delta{=}0.05, \alpha_{\text{sig}}{=}0.05)$-reliability:} fraction of trials with $|\hat F_m - F^{*}| \leq 0.05$.
\end{itemize}

We report on six (dataset, skew-level) configurations spanning $\pi_{\max} \in [0.26, 0.96]$ (Table~\ref{tab:rq2_real}): two at natural prevalence (\texttt{cebab\_stars}, \texttt{summeval}) and four skew-masked at threshold $t$ via the class-merging protocol of \S\ref{sec:setup}.

\begin{table}[!htbp]
\centering
\caption{\textbf{Test-independent design-quality metrics on real benchmarks.} R=\textsc{Random}, S=\textsc{Strat}, SR=\textsc{Seq-Refined}, SC=\textsc{Seq-Coverage}. Lower is better for bias/std; higher for reliability; bold marks the best design per row.}
\label{tab:rq2_real}
\scriptsize
\setlength{\tabcolsep}{3pt}
\begin{tabular}{llrrrrrrrrrrrrr}
\toprule
& & \multicolumn{4}{c}{$|\mathrm{bias}|$ of $\hat F_m$} & \multicolumn{4}{c}{$\mathrm{std}(\hat F_m)$} & \multicolumn{4}{c}{Rel.\ ($\delta = 0.05$)} \\
\cmidrule(lr){3-6} \cmidrule(lr){7-10} \cmidrule(lr){11-14}
Dataset ($\pi_{\max}$) & $\rho$ & R & S & SR & SC & R & S & SR & SC & R & S & SR & SC \\
\midrule
\multirow{3}{*}{cebab\_stars (.26)} & .05 & .087 & \textbf{.001} & .087 & .105 & .126 & \textbf{.074} & .124 & .117 & .22 & \textbf{.49} & .24 & .22 \\
 & .10 & .069 & \textbf{.004} & .070 & .123 & .080 & \textbf{.051} & .079 & .079 & .33 & \textbf{.64} & .32 & .13 \\
 & .25 & .039 & \textbf{.001} & .044 & .049 & .040 & \textbf{.028} & .040 & .042 & .58 & \textbf{.93} & .55 & .51 \\
\midrule
\multirow{3}{*}{summeval (.56)} & .05 & .020 & .034 & \textbf{.016} & .016 & .030 & \textbf{.026} & .026 & .027 & .82 & .70 & .88 & \textbf{.90} \\
 & .10 & .014 & .031 & .013 & \textbf{.012} & .019 & .018 & .018 & \textbf{.017} & .95 & .84 & .97 & \textbf{.98} \\
 & .25 & .008 & .021 & .008 & \textbf{.001} & .011 & .010 & .010 & \textbf{.010} & \textbf{1.00} & .99 & 1.00 & 1.00 \\
\midrule
\multirow{3}{*}{lesion@$t{=}0$ (.65)} & .05 & .044 & \textbf{.018} & .045 & .043 & .081 & .101 & .080 & \textbf{.071} & .36 & .37 & .41 & \textbf{.45} \\
 & .10 & .034 & \textbf{.002} & .039 & .049 & .052 & .071 & .053 & \textbf{.047} & \textbf{.57} & .53 & .53 & .51 \\
 & .25 & .020 & \textbf{.002} & .024 & .038 & .032 & .040 & .033 & \textbf{.025} & \textbf{.82} & .77 & .77 & .66 \\
\midrule
\multirow{3}{*}{lesion@$t{=}1$ (.74)} & .05 & .130 & \textbf{.021} & .131 & .152 & .088 & .105 & .088 & \textbf{.084} & .16 & \textbf{.35} & .19 & .10 \\
 & .10 & .111 & \textbf{.003} & .113 & .148 & .059 & .068 & .055 & \textbf{.055} & .13 & \textbf{.53} & .16 & .03 \\
 & .25 & .072 & \textbf{.006} & .074 & .122 & .035 & .038 & .032 & \textbf{.024} & .26 & \textbf{.78} & .24 & .00 \\
\midrule
\multirow{3}{*}{lesion@$t{=}2$ (.93)} & .05 & .116 & .129 & .111 & \textbf{.105} & .176 & .264 & .173 & \textbf{.149} & .17 & .16 & .19 & \textbf{.20} \\
 & .10 & .076 & .088 & .070 & \textbf{.060} & .123 & .165 & .113 & \textbf{.089} & .29 & .18 & .34 & \textbf{.40} \\
 & .25 & .035 & .055 & .030 & \textbf{.021} & .066 & .080 & .058 & \textbf{.045} & .51 & .41 & .60 & \textbf{.70} \\
\midrule
\multirow{3}{*}{cebab\_aspects@$t{=}0$ (.96)} & .05 & .135 & .154 & \textbf{.124} & .269 & .374 & \textbf{.179} & .359 & .408 & .17 & \textbf{.18} & \textbf{.18} & .04 \\
 & .10 & \textbf{.081} & .082 & .094 & .278 & .231 & \textbf{.129} & .210 & .242 & .19 & \textbf{.31} & .19 & .09 \\
 & .25 & .038 & \textbf{.029} & .037 & .044 & .100 & \textbf{.065} & .096 & .102 & .36 & \textbf{.53} & .42 & .36 \\
\bottomrule
\end{tabular}
\end{table}

\textbf{Strat cuts estimator bias by an order of magnitude when the strata are informative.}
On \texttt{cebab\_stars} ($\pi_{\max}=0.26$), \textsc{Strat} reduces $|\mathrm{bias}(\hat F_m)|$ by roughly $17$--$87\times$ relative to \textsc{Random} across all $\rho$, lifting reliability from 22--58\% to 49--93\%. On \texttt{lesion}@$t{=}1$ ($\pi_{\max}=0.74$) the pattern repeats: $6$--$37\times$ bias reduction, 19--52pp reliability gains. On \texttt{summeval} ($\pi_{\max}=0.56$), all designs perform similarly because the marginal-mismatch bias is already small.

\textbf{Adaptive strategies are not a uniformly better default.}
\textsc{Seq-Refined} matches \textsc{Strat} only when its between-rater consensus signal is itself accurate; \textsc{Seq-Coverage}'s coverage-equalising tiebreak helps in some cells but amplifies variance when coverage diversity conflicts with representative sampling. We therefore recommend the simpler \textsc{Strat} design when the primary stratum signal has usable variation, and sequential coverage only when that signal collapses.

\textbf{Limitation: Strat requires informative strata.}
At $\pi_{\max}\geq 0.93$, the primary annotator labels nearly all items identically and stratification collapses; \textsc{Seq-Coverage} or higher $\rho$ become necessary (Appendix~\ref{app:strat_guide}).

\subsection{RQ3: How Do Real LLM Judges Behave Under Sparsity?}
\label{sec:rq3}

We apply the pipeline to 10 LLM judges and four real-world evaluation matrices (WAX, CeBaB-aspects, CeBaB-stars, SummEval; \S\ref{sec:setup}), asking whether synthetic patterns persist on real data.

\textbf{Sparsity scrambles model rankings.}
Beyond deploy/reject, practitioners often need to \emph{rank} candidate judges. Table~\ref{tab:rq3_ranking} reports ranking stability across 10 judges per benchmark: at $\rho = 0.05$ the mean top-1 error is 65\% (CeBaB-aspects/stars near 88\%; SummEval at 18\% thanks to its 6\,400-item corpus), and at $\rho = 0.25$ remains 36\%. Ranking demands stricter overlap than deploy/reject because pairwise margins between candidate judges are tighter than the human--judge gap (Thm.~\ref{thm:planning}(iii)).

\begin{table}[!htbp]
\centering
\caption{\textbf{Ranking stability under sparsity.} ``Rank err.''\ = fraction of pairwise rankings scrambled; ``Top-1''\ = probability of selecting the wrong best judge. 10 judges per benchmark.}
\label{tab:rq3_ranking}
\scriptsize
\setlength{\tabcolsep}{3pt}
\begin{tabular}{lrrrrrrrrrr}
\toprule
& \multicolumn{2}{c}{WAX} & \multicolumn{2}{c}{CeBaB-asp.} & \multicolumn{2}{c}{CeBaB-stars} & \multicolumn{2}{c}{SummEval} & \multicolumn{2}{c}{\textbf{Mean}} \\
\cmidrule(lr){2-3} \cmidrule(lr){4-5} \cmidrule(lr){6-7} \cmidrule(lr){8-9} \cmidrule(lr){10-11}
$\rho$ & Rank & Top-1 & Rank & Top-1 & Rank & Top-1 & Rank & Top-1 & Rank & Top-1 \\
\midrule
0.05 & .297 & .670 & .349 & .887 & .342 & .875 & .122 & .177 & .278 & .652 \\
0.10 & .234 & .600 & .279 & .777 & .255 & .787 & .083 & .080 & .213 & .561 \\
0.25 & .146 & .427 & .196 & .507 & .156 & .517 & .049 & .003 & .137 & .363 \\
0.50 & .071 & .270 & .138 & .213 & .085 & .350 & .031 & .000 & .082 & .208 \\
0.75 & .031 & .103 & .091 & .033 & .053 & .110 & .018 & .000 & .048 & .062 \\
\bottomrule
\end{tabular}
\end{table}

\textbf{Stratified sampling reduces false rejection; false approval depends on judge quality.}
Table~\ref{tab:rq3_real} aggregates wrong-decision rates across all four evaluation matrices (40~judges; per-judge detail in Appendix~\ref{app:real_strat}). At $\rho = 0.05$, \textsc{Strat} halves the mean FR rate from 29\% to 14\%; by $\rho = 0.25$, \textsc{Strat} falls below 3\% FR, while the other designs still sit around 11--16\%. FA rates are lower (7--17\% at $\rho = 0.05$) because most reject judges sit far below the boundary, but FA does \emph{not} decay to zero even at $\rho = 0.50$ (7--11\%), motivating the borderline-judge analysis below.

\begin{table}[!htbp]
\centering
\caption{\textbf{Wrong-decision rates by design and overlap} (aggregated across the four real evaluation matrices). FR = mean false-rejection rate on 20 pass-judges ($\omega \geq 0.60$); FA = mean false-approval rate on 16 reject judges ($\omega < 0.50$); WDR = mean over all 40 judges. Bold = best design per $\rho$.}
\label{tab:rq3_real}
\scriptsize
\setlength{\tabcolsep}{3pt}
\begin{tabular}{llrrr}
\toprule
$\rho$ & Design & Mean FR & Mean FA & Mean WDR \\
\midrule
\multirow{4}{*}{0.05}
  & \textsc{Random} & .290 & .112 & .250 \\
  & \textsc{Strat} & \textbf{.142} & .172 & \textbf{.184} \\
  & \textsc{Seq-Refined} & .308 & .098 & .250 \\
  & \textsc{Seq-Coverage} & .435 & \textbf{.066} & .311 \\
\midrule
\multirow{4}{*}{0.10}
  & \textsc{Random} & .238 & .082 & .212 \\
  & \textsc{Strat} & \textbf{.087} & .136 & \textbf{.137} \\
  & \textsc{Seq-Refined} & .230 & .086 & .209 \\
  & \textsc{Seq-Coverage} & .328 & \textbf{.072} & .264 \\
\midrule
\multirow{4}{*}{0.25}
  & \textsc{Random} & .122 & .071 & .149 \\
  & \textsc{Strat} & \textbf{.027} & .114 & \textbf{.092} \\
  & \textsc{Seq-Refined} & .113 & .070 & .146 \\
  & \textsc{Seq-Coverage} & .158 & \textbf{.075} & .174 \\
\midrule
\multirow{4}{*}{0.50}
  & \textsc{Random} & .021 & .079 & .089 \\
  & \textsc{Strat} & \textbf{.007} & .110 & \textbf{.077} \\
  & \textsc{Seq-Refined} & .027 & .077 & .092 \\
  & \textsc{Seq-Coverage} & .030 & \textbf{.075} & .096 \\
\bottomrule
\end{tabular}
\end{table}

\textsc{Strat} lowers FR but raises FA at $\rho=0.05$ (14.2\% vs 29.0\% FR; 17.2\% vs 11.2\% FA), because correcting the human-ceiling upward bias lets some borderline-reject judges through. When FA cost dominates: increase $\rho$ or use \textsc{Seq-Coverage} (Appendix~\ref{app:fa_guide}).

\textbf{Borderline judges reveal a fundamental limit that higher overlap mitigates.}
Table~\ref{tab:rq3_borderline} isolates four judges sitting exactly on the pass/reject boundary ($\omega = 0.50$). Under \textsc{Strat}, strong pass-judges converge below 3\% WDR by $\rho = 0.25$, but borderline judges remain at 30--44\% WDR even under the best design: when $\omega$ equals the threshold, any finite-sample estimator has ${\sim}50\%$ error. Practitioners should budget $\rho \geq 0.25$ for non-borderline judges and treat verdicts with $\hat\omega$ near $0.5$ as inconclusive.

\begin{table}[!htbp]
\centering
\caption{\textbf{Borderline vs.\ strong judges across designs:} mean wrong-decision rate at increasing $\rho$ (aggregated over the four real evaluation matrices). Strong pass ($\omega \geq 0.60$, 20 judges), borderline pass ($\omega = 0.50$, 4 judges), reject ($\omega < 0.50$, 16 judges). Bold = best design per category and $\rho$.}
\label{tab:rq3_borderline}
\scriptsize
\setlength{\tabcolsep}{3pt}
\begin{tabular}{llrrrr}
\toprule
Design & Category & $\rho = 0.05$ & $0.10$ & $0.25$ & $0.50$ \\
\midrule
\multirow{3}{*}{\textsc{Random}}
  & Strong pass (20) & .290 & .238 & .122 & .021 \\
  & Borderline pass (4) & .597 & .597 & .599 & .464 \\
  & Reject (16) & .112 & .082 & .071 & .079 \\
\midrule
\multirow{3}{*}{\textsc{Strat}}
  & Strong pass (20) & \textbf{.142} & \textbf{.087} & \textbf{.027} & \textbf{.007} \\
  & Borderline pass (4) & \textbf{.440} & \textbf{.389} & \textbf{.323} & \textbf{.295} \\
  & Reject (16) & .172 & .136 & .114 & .110 \\
\midrule
\multirow{3}{*}{\textsc{Seq-Refined}}
  & Strong pass (20) & .308 & .230 & .113 & .027 \\
  & Borderline pass (4) & .572 & .602 & .618 & .482 \\
  & Reject (16) & .098 & .086 & \textbf{.070} & .077 \\
\midrule
\multirow{3}{*}{\textsc{Seq-Coverage}}
  & Strong pass (20) & .435 & .328 & .158 & .030 \\
  & Borderline pass (4) & .676 & .709 & .647 & .509 \\
  & Reject (16) & \textbf{.066} & \textbf{.072} & .075 & \textbf{.075} \\
\bottomrule
\end{tabular}
\end{table}

\textbf{Real-world tests align with synthetic findings.}
The elevated wrong-decision rates do not contradict RQ1: real judges exhibit heterogeneous per-rater agreement (e.g.\ Gemini Flash on CeBaB-stars: per-rater $\hat p_o$ ranges 0.40--0.64 against a human baseline of ${\sim}0.51$), so their deltas cluster near the $\varepsilon$-threshold---a regime the i.i.d.\ synthetic grid does not probe. The pipeline still separates strong from weak judges, but borderline cases require higher overlap or should be flagged as inconclusive (correlated-item robustness in Appendix~\ref{app:correlated}).

\subsection{RQ4: How Much Overlap Is Enough?}
\label{sec:rq4}

Having established that design choice and overlap jointly control wrong-decision rates (\S\ref{sec:rq2}--\S\ref{sec:rq3}), we next translate these findings into concrete planning rules.

\textbf{Required overlap for a precision target.}
Table~\ref{tab:overlap_required} reports the empirical minimum $\rho$ at which the human-pool score $\hat F_m$ (computed with $\hat p_o$ under per-rater sparse overlap) stays within $\pm\delta$ of its dense reference on $\geq$95\% of 300 synthetic trials. For the common $\delta=0.05$ target, practitioners need 25\% overlap under uniform prevalence ($L=2$), widening to 15--50\% under skew ($L=5$); at $\rho=0.05$ achievable precision is only $\delta=0.10$, large enough to flip a deployment decision.

\begin{table}[!htbp]
\centering
\scriptsize
\caption{\textbf{Minimum overlap for a precision target} on the human-pool score $\hat F_m$ using $\hat p_o$ under per-rater sparse overlap ($n=500$, $K=4$, $F^{*} \in [0.6, 0.9]$, 300 trials, $\alpha_{\text{sig}}=0.05$). Ranges span the target $F^{*}$ values.}
\label{tab:overlap_required}
\begin{tabular*}{\columnwidth}{@{\extracolsep{\fill}}llrrr}
\toprule
& & \multicolumn{3}{c}{Min.\ $\rho$ for $\Pr[|\hat F_m - F^{*}| > \delta] \leq 0.05$} \\
\cmidrule(lr){3-5}
$L$ & Prevalence & $\delta\!=\!0.02$ & $\delta\!=\!0.05$ & $\delta\!=\!0.10$ \\
\midrule
\multirow{2}{*}{2} & Uniform & $\geq$75\% & 25\% & 5--10\% \\
 & Skewed (.8/.2) & $\geq$75\% & 25\% & 5--10\% \\
\midrule
\multirow{2}{*}{5} & Uniform & $\geq$75\% & 15--25\% & 5--10\% \\
 & Skewed (.7/.1/.1/.05/.05) & $\geq$75\% & 15--50\% & 5--10\% \\
\bottomrule
\end{tabular*}
\end{table}

\textbf{Overlap dominates rater count.}
Figure~\ref{fig:k_effect} shows diminishing returns from adding raters relative to increasing overlap. At $\rho = 0.05$, doubling $K$ from 2 to 10 lifts reliability only from 17\% to 40\%, whereas raising $\rho$ from 0.05 to 0.25 at $K = 3$ jumps reliability from 27\% to 63\%. Curves flatten after $K = 5$ while each $\rho$ step yields a larger vertical jump under both uniform and skewed prevalence, identifying $K = 3\text{--}5$ with $\rho \geq 0.25$ as the practical sweet spot.

\begin{figure}[!htbp]
\centering
\includegraphics[width=0.55\columnwidth]{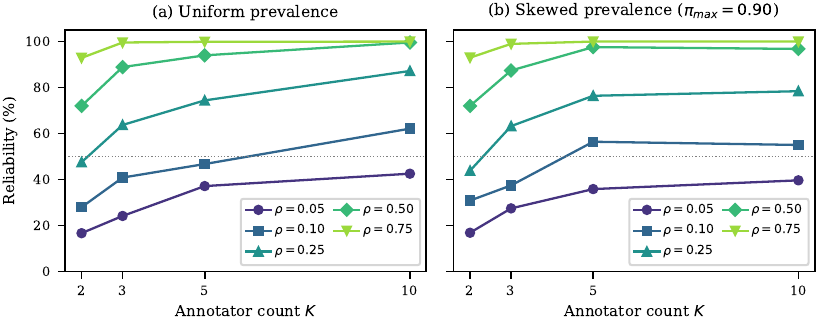}
\caption{\textbf{Reliability vs.\ annotator count at five overlap rates} ($n = 500$, $p_h = 0.80$). Reliability = \% of trials within $\pm 0.05$ of ground truth. Grey line marks 50\%.}
\label{fig:k_effect}
\end{figure}

\section{Conclusion}
\label{sec:conclusion}

We studied LLM-judge validation when only a small fraction of items receive multiple human labels. Sparse overlap injects $\gamma_F(\pi)/m$ variance into per-rater agreement comparisons (Thm.~\ref{thm:sparsity}), with $\gamma_F$ diverging for $\hat\kappa$ and $\hat\alpha$ under label skew; among the design levers ($F$, $\rho$, $\mathbf{D}$), overlap quantity is the necessary condition for reliable decisions. The two actionable levers are overlap \emph{quantity}, planned via recipe~\eqref{eq:overlap_recipe} with certification and ranking variants (Thm.~\ref{thm:planning}), and overlap \emph{allocation}, where \textsc{Strat} cuts estimator bias by an order of magnitude when strata are informative. Practical recipe: $F = \hat p_o$ with \textsc{Strat}, $\rho \geq 0.25$, $K = 3$--$5$ raters, extra overlap for ranking. Borderline judges ($\omega \approx 0.50$) represent a fundamental limit and should be flagged as inconclusive (Table~\ref{tab:rq3_borderline}).

\textbf{Limitations and future work.} Results assume i.i.d.\ items and batched annotation; clustered item difficulty demands cluster-aware stratification (Appendix~\ref{app:correlated}). \textsc{Strat} requires informative strata and degrades under extreme skew ($\pi_{\max} > 0.90$). Recipe~\eqref{eq:overlap_recipe} is asymptotic and needs a pre-collection variance estimate. The overlap-mask-independent-of-labels assumption (Thm.~\ref{thm:sparsity}) requires that annotation assignments are random or coverage-based, not self-selected. Future work covers adaptive within-rater policies, preference-based and rubric-based evaluation (which face the same sparse-estimation bottleneck), and non-deterministic judges (whose additional variance makes our overlap targets conservative lower bounds).

\bibliography{references}
\bibliographystyle{plainnat}

\clearpage
\appendix
%

\footnotesize

\section{Formal Proofs}
\label{app:proofs}

\subsection{Coefficient Definitions}
\label{app:coef_forms}

We collect the explicit sparse-data forms of the four agreement coefficients used throughout the paper (\S\ref{sec:agreement_coefficients}). All sums run only over items selected by the design mask $\mathbf{D}$; when $\rho = 1$ every form reduces to its standard textbook definition \citep{krippendorff2004content,cohen1960coefficient,gwet2008computing}.

\textbf{Pairwise observed agreement $\hat p_o$.}
For rater pair $(k,k')$, the overlap set is $\calI_{kk'} = \{i : D_{ik} = D_{ik'} = 1\}$, with $m_{kk'} = |\calI_{kk'}|$. The pairwise observed agreement is
\[
\hat p_o^{(kk')} = \frac{1}{m_{kk'}}\sum_{i \in \calI_{kk'}} \mathbf{1}[A_{ik} = A_{ik'}].
\]
No chance correction is applied: the estimator is simply the fraction of co-annotated items on which the two raters agree. This makes $\hat p_o$ unbiased with variance $p_o(1-p_o)/m_{kk'}$ (Thm.~\ref{thm:sparsity}(i)) and is why $\hat p_o$ serves as the default scoring function.

\textbf{Krippendorff's $\hat\alpha$} pools all within-item pairs across all raters into a single disagreement ratio:
\[
\hat\alpha = 1 - \frac{\hat D_o}{\hat D_e}, \quad
\hat D_o = \frac{\sum_{i} \sum_{k\neq k': D_{ik}D_{ik'}=1} \delta(A_{ik},A_{ik'})}{\sum_i n_i(n_i-1)}, \quad
\hat D_e = \frac{\sum_{\ell,\ell'} N_\ell N_{\ell'}\,\delta(c_\ell,c_{\ell'})}{N(N-1)},
\]
with $n_i = \sum_k D_{ik}$ (number of raters on item $i$), $\delta(c,c') = \mathbf{1}[c\neq c']$ for nominal data, $N = \sum_{i,k} D_{ik}$ (total label count), and $N_\ell = \sum_{i,k} D_{ik}\mathbf{1}[A_{ik}=c_\ell]$ (total count of category $\ell$). We write $\hat\alpha_{\text{dense}}$ when $\rho = 1$ and $\hat\alpha_m$ when the dependence on the multi-annotated count $m = |\calI^+|$ matters; the estimator is the same in both cases, only the mask distribution differs. As a ratio statistic, $\hat\alpha$ inherits an $O(1/m)$ bias from the random denominator $\hat D_e$ (proof in \S\ref{app:proof_sparsity}(ii)), and its amplification factor $\gamma_\alpha = O\bigl((1-p_e^\alpha)^{-2}\bigr)$ grows under skew (Thm.~\ref{thm:sparsity}(ii)).

\textbf{Cohen's pairwise $\hat\kappa$.}
For rater pair $(k,k')$ on the overlap set $\calI_{kk'}$:
\[
\hat\kappa^{(kk')} = \frac{\hat p_o^{(kk')} - \hat p_e^{(kk')}}{1 - \hat p_e^{(kk')}}, \qquad
\hat p_e^{(kk')} = \sum_{\ell=1}^{L} \hat\pi_\ell^{(k)}\,\hat\pi_\ell^{(k')},
\]
where the per-rater marginals are computed on the overlap set: $\hat\pi_\ell^{(k)} = m_{kk'}^{-1}\sum_{i \in \calI_{kk'}} \mathbf{1}[A_{ik}=c_\ell]$. The denominator $1 - p_e$ measures how much room exists above chance; for binary labels with marginals $(\pi_{\max}, 1-\pi_{\max})$ this simplifies to $1 - p_e = 2\pi_{\max}(1-\pi_{\max})$, which vanishes as $\pi_{\max} \to 1$. This is the \emph{prevalence paradox} \citep{feinstein1990high,byrt1993bias}: even when raters agree on nearly every item, $\hat\kappa$ can be low (or its variance can explode) because the chance-correction denominator shrinks. Under sparse overlap, $\gamma_\kappa = O((1-p_e^\kappa)^{-4})$ (Thm.~\ref{thm:sparsity}(iii)), making $\hat\kappa$ the most noise-sensitive of the four coefficients.

\textbf{Gwet's $\widehat{\mathrm{AC1}}$} \citep{gwet2008computing}.
AC1 shares $\hat\kappa$'s structure but replaces the rater-marginal chance term with a symmetric alternative:
\[
\widehat{\mathrm{AC1}}^{(kk')} = \frac{\hat p_o^{(kk')} - \hat p_e^{\mathrm{AC1}}}{1 - \hat p_e^{\mathrm{AC1}}}, \qquad
\hat p_e^{\mathrm{AC1}} = \frac{1}{L-1}\sum_{\ell=1}^{L} \hat\pi_\ell(1-\hat\pi_\ell),
\]
where $\hat\pi_\ell$ is the pooled marginal across both raters on $\calI_{kk'}$. The key structural property is that $1 - p_e^{\mathrm{AC1}} \geq (L-1)/L \geq 1/2$ for all $L \geq 2$ and all prevalence vectors $\pi$ (shown in \S\ref{app:proof_sparsity}(iv)). The denominator constants are therefore bounded, making $\widehat{\mathrm{AC1}}$ immune to the prevalence paradox. The cost is a different interpretation of ``chance agreement'': AC1's correction reflects category-count entropy rather than rater-specific labelling tendencies.

\subsection{Proof of Theorem~\ref{thm:sparsity}}
\label{app:proof_sparsity}

\begin{proof}
We prove each part of the theorem in turn.

\medskip\noindent\textbf{Part (i): $\hat p_o$.}\;
The pairwise observed agreement on an overlap subset of $m$ items is $\hat p_{o,m} = m^{-1}\sum_{i=1}^{m}\mathbf{1}[A_{ik}=A_{ik'}]$, a sample mean of i.i.d.\ $\mathrm{Bernoulli}(p_o)$ indicators (under the i.i.d.\ item assumption and overlap mask independent of labels). Therefore $\E[\hat p_{o,m}] = p_o$ (unbiased) and $\Var(\hat p_{o,m}) = p_o(1-p_o)/m = O(1/m)$.

\medskip\noindent\textbf{Part (ii): $\hat\alpha$.}\;
$\hat{\alpha}_m = 1 - \hat{D}_o / \hat{D}_e$ is a ratio statistic.
Let $\calI^+ = \{i \in \calI : \sum_k D_{ik} \geq 2\}$ denote the multi-annotated subset and $m = |\calI^+|$.
Define
\[
w_i = n_i(n_i-1), \qquad d_i = \frac{1}{w_i}\sum_{k \neq k'} \delta(A_{ik}, A_{ik'}).
\]
Under the balanced-overlap assumption, $w_i \equiv w$ on $\calI^+$.
Therefore $\sum_{i \in \calI^+} w_i = mw$ and $\sum_{i \in \calI^+} w_i d_i = w\sum_{i \in \calI^+} d_i$, so
\[
\hat{D}_o = \frac{\sum_{i \in \calI^+} w_i d_i}{\sum_{i \in \calI^+} w_i} = \frac{w\sum_{i \in \calI^+} d_i}{mw} = \frac{1}{m}\sum_{i \in \calI^+} d_i.
\]
Because $\calI^+$ is sampled uniformly and independently of the labels, the $d_i$ in the overlap subset are i.i.d.\ draws from the population item-level disagreement distribution, so $\hat{D}_o$ is an unbiased estimator of $D_o$ with variance $O(1/m)$.

For the denominator, recall the population label marginal $\pi_\ell = \Pr(A_{i,k} = c_\ell)$ from the Notation paragraph (well-defined under the rater-exchangeability of $P$), so
\[
D_e = \sum_{\ell=1}^{L}\sum_{\ell'=1}^{L} \pi_\ell\,\pi_{\ell'}\,\delta(c_\ell, c_{\ell'}).
\]
Under the i.i.d.\ superpopulation assumption, the empirical marginal vector $\hat{\pi}_\ell = N_\ell/N$ is a sample proportion with $\Cov(\hat{\pi}) = \Sigma/n$ for a fixed population covariance matrix~$\Sigma$.
Since $D_e$ is a smooth (quadratic) function of~$\pi$, the multivariate delta method \citep[Theorem~3.1]{vandervaart1998asymptotic} gives
\[
\Var(\hat{D}_e) = \nabla D_e^\top \frac{\Sigma}{n}\,\nabla D_e + O(1/n^2) = O(1/n),
\]
with bias $O(1/n)$ from the second-order term.
No annotator model is needed: the variance bound follows from the i.i.d.\ item assumption and the standard result for sample proportions.

Now apply the bivariate delta method to the smooth map $h(x,y) = 1 - x/y$ around $(D_o, D_e)$, with $D_e > 0$ by assumption.
The partial derivatives are $\partial_x h = -1/D_e$ and $\partial_y h = D_o/D_e^2$.
The first-order delta method gives the variance:
\[
\Var(\hat{\alpha}_m) = \frac{\Var(\hat{D}_o)}{D_e^2} + \frac{D_o^2\,\Var(\hat{D}_e)}{D_e^4} + O(1/m^2) = O(1/m) + O(1/n) + O(1/m^2) = O(1/m)
\]
since $m \leq n$, where the $O(1/m^2)$ remainder collects all higher-order terms from the Taylor expansion.
The second-order terms give the bias:
even when $\hat{D}_o$ and $\hat{D}_e$ are individually unbiased to first order, the ratio $\hat{D}_o/\hat{D}_e$ satisfies $\E[\hat{D}_o/\hat{D}_e] \neq \E[\hat{D}_o]/\E[\hat{D}_e]$ because $x \mapsto 1/x$ is convex (Jensen's inequality applied to the random denominator), producing a second-order bias of size $O(\Var(\hat{D}_e)/D_e^2) = O(1/n)$ from the denominator and $O(\Var(\hat{D}_o)/D_e) = O(1/m)$ from the numerator, giving
\[
\E[\hat{\alpha}_m] - \alpha = O(1/m), \qquad \Var(\hat{\alpha}_m) = O(1/m),
\]
with constants depending on $D_e^{-1}$.

\medskip\noindent\textbf{Part (iii): $\hat\kappa$.}\;
The population target is $\kappa = g(p_o, p_e)$ and the sparse estimator is $\hat{\kappa}_m = g(\hat{p}_o, \hat{p}_e)$, where $g(x,y) = (x-y)/(1-y)$.
The gradient of $g$, evaluated at the population values $(p_o, p_e)$, is $\nabla g = (1/(1-p_e),\; (p_o - 1)/(1-p_e)^2)$.
By the multivariate delta method:
\[
    \Var(\hat{\kappa}) \approx \frac{\Var(\hat{p}_o)}{(1-p_e)^2} + \frac{(p_o - 1)^2\,\Var(\hat{p}_e)}{(1-p_e)^4} - \frac{2(p_o-1)\,\Cov(\hat{p}_o, \hat{p}_e)}{(1-p_e)^3}
\]
Under overlap sampling over $m$ items, $\Var(\hat{p}_o)$, $\Var(\hat{p}_e)$, and $\Cov(\hat{p}_o, \hat{p}_e)$ are each $O(1/m)$.
Therefore the variance remains $O(1/m)$ for fixed $p_e<1$, but the denominator constants include powers of $(1-p_e)^{-1}$; a conservative bound is
\[
\Var(\hat{\kappa}_m) = O\!\left(\frac{1}{m(1-p_e)^4}\right).
\]
This holds for any $L \geq 2$. For binary labels with symmetric marginals, $1 - p_e = 2\pi_{\max}(1-\pi_{\max})$, so the denominator constants grow as prevalence concentrates.
In general, as $\pi_{\max} \to 1$ we have $p_e \to 1$ for any $L$, so the variance constant grows without bound.

\medskip\noindent\textbf{Part (iv): $\widehat{\mathrm{AC1}}$.}\;
The population target is $\mathrm{AC1} = g(p_o, p_e^{\mathrm{AC1}})$ with $g(x,y) = (x-y)/(1-y)$, identical in form to $\kappa$; only the chance-correction term differs: $p_e^{\mathrm{AC1}} = \tfrac{1}{L-1}\sum_\ell \pi_\ell(1-\pi_\ell)$.
The delta-method variance has the same structure as Part~(iii):
\[
\Var(\widehat{\mathrm{AC1}}_m) = O\!\left(\frac{1}{m\,(1 - p_e^{\mathrm{AC1}})^4}\right).
\]
The critical difference is the denominator. For $L$ labels with any prevalence vector~$\pi$:
\[
1 - p_e^{\mathrm{AC1}} = 1 - \frac{1}{L-1}\sum_\ell \pi_\ell(1-\pi_\ell) = 1 - \frac{1 - \sum_\ell \pi_\ell^2}{L-1} = \frac{L - 2 + \sum_\ell \pi_\ell^2}{L-1}.
\]
Since $\sum_\ell \pi_\ell^2 \geq 1/L$ (by Cauchy--Schwarz) and $L \geq 2$:
\[
1 - p_e^{\mathrm{AC1}} \;\geq\; \frac{L - 2 + 1/L}{L-1} \;=\; \frac{(L-1)^2}{L(L-1)} \;=\; \frac{L-1}{L} \;\geq\; \frac{1}{2}.
\]
Hence $(1 - p_e^{\mathrm{AC1}})^{-4} \leq 16$ for all $L \geq 2$ and all $\pi$, giving $\Var(\widehat{\mathrm{AC1}}_m) = O(1/m)$ with a prevalence-independent denominator constant.
This is why $\widehat{\mathrm{AC1}}$ neutralises the prevalence paradox that afflicts $\hat\kappa$.
\end{proof}

\subsection{Proof of Theorem~\ref{thm:planning}}
\label{app:proof_planning}

\begin{proof}
Under the i.i.d.\ item assumption (Theorem~\ref{thm:sparsity}) and the observed-agreement scoring rule used in Theorem~\ref{thm:planning}, the per-rater score difference $d_k = S(j;r_k) - S(r_{-k};r_k)$ is a sample mean of $m$ i.i.d.\ item-level terms $\xi_i = \mathbf{1}[A_{ij}=A_{i,\text{held-in}}] - \mathbf{1}[A_{ir_k}=A_{i,\text{held-in}}]$, each with finite mean and variance. Set $\mu_d := \E[\xi_i] = \E[d_k]$ and $\sigma_0^2 := \Var(\xi_i)$. Then $\Var(d_k) = \sigma_0^2/m$ and the central limit theorem gives the approximation $d_k \approx N(\mu_d, \sigma_0^2/m)$ for large enough $m$. For chance-corrected ratio coefficients, the same planning form can be used as a delta-method approximation after estimating the corresponding variance constant, but those ratios are not itemwise averages.

\medskip\noindent\textbf{Part (i).}\;
The per-rater outcome (Eq.~\eqref{eq:per_rater}) is $T_k = \mathbf{1}[d_k \geq -\varepsilon]$, so a false rejection occurs iff $d_k < -\varepsilon$. Standardising under the Gaussian approximation:
\[
\Pr(d_k < -\varepsilon) \;\approx\; \Phi\!\left(\frac{-\varepsilon - \mu_d}{\sigma_0/\sqrt{m}}\right) \;=\; \Phi\!\left(\frac{-(\mu_d + \varepsilon)\sqrt{m}}{\sigma_0}\right).
\]

\medskip\noindent\textbf{Part (ii).}\;
For $\mu_d + \varepsilon > 0$, the right-hand side of (i) is strictly decreasing in $m$. Solving the approximate inequality $\Phi\!\bigl(-(\mu_d+\varepsilon)\sqrt{m}/\sigma_0\bigr) \leq \alpha_{\text{sig}}/2$ gives the sufficient planning rule
\[
\frac{(\mu_d + \varepsilon)\sqrt{m}}{\sigma_0} \;\geq\; z_{1-\alpha_{\text{sig}}/2}, \qquad \text{so} \qquad m \;\geq\; \frac{z_{1-\alpha_{\text{sig}}/2}^{\,2}\,\sigma_0^2}{(\mu_d + \varepsilon)^2}.
\]
We use the minimum integer overlap $m^{*}_{\text{cert}} = \lceil z_{1-\alpha_{\text{sig}}/2}^{\,2}\,\sigma_0^2 / (\mu_d+\varepsilon)^2 \rceil$ as an asymptotic planning target.

\medskip\noindent\textbf{Part (iii).}\;
To rank $J$ judges, consider the best judge $j^*$ with population score $S_{j^*} \geq S_j$ for all $j$.
Under sparse overlap, estimated scores are approximately Gaussian: $\hat S_j \approx N(S_j, \sigma_j^2/m)$.
With approximately equal per-judge variance $\sigma_j^2 \approx \sigma_0^2$, the estimated score difference between the best judge and any competitor $j$ satisfies
\[
\hat S_{j^*} - \hat S_j \;\approx\; N\!\bigl(\Delta_{j^*\!j},\; 2\sigma_0^2/m\bigr), \qquad \Delta_{j^*\!j} := S_{j^*} - S_j.
\]
A top-1 error occurs when any competitor's estimated score exceeds $\hat S_{j^*}$. By a union bound over $J-1$ competitors:
\[
\Pr[\text{top-1 error}] \;\leq\; (J{-}1)\,\Phi\!\Bigl(-\Delta_{\min}\sqrt{m/(2\sigma_0^2)}\Bigr),
\]
where $\Delta_{\min} = \min_{j \neq j^*} \Delta_{j^*\!j}$. Setting this $\leq \alpha$ and solving:
\[
m \;\geq\; m^{*}_{\text{rank}} \;=\; \Bigl\lceil 2\,z_{1-\alpha/(J-1)}^{\,2}\,\sigma_0^2\,/\,\Delta_{\min}^2 \Bigr\rceil.
\]
Comparing $m^*_{\text{rank}}$ to $m^*_{\text{cert}}$, three effects compound: (a)~the denominator shrinks from $(\mu_d+\varepsilon)^2$ to $\Delta_{\min}^2$ (pairwise judge margins are tighter than the human--judge gap); (b)~the critical value inflates from $z_{1-\alpha_{\text{sig}}/2}$ to $z_{1-\alpha/(J-1)}$ (Bonferroni over $J-1$ comparisons); (c)~a factor of~2 appears from comparing two noisy estimates. Together these explain the empirical finding in \S\ref{sec:rq3} that ranking 10 judges needs substantially higher $\rho$ than certifying any single one.
\end{proof}

\section{Implementation and Data Details}
\label{app:impl_data}

\subsection{Motivating Example: Depth vs.\ Breadth on ISIC Lesion}
\label{app:depth_breadth}

Figure~\ref{fig:depth_breadth} in the introduction presents the paper's motivating example.
We use the ISIC Lesion benchmark \citep{codella2018skin}: 488 evaluation items (dermoscopic feature ratings across skin lesion images) with 6 student annotators and 4 LLM judges.
All annotators rated all items on a 0--6 severity scale, giving a dense human-agreement reference of $\hat{\alpha}_{\text{dense}} = 0.351$.
We select 3 of the 6 annotators and simulate a realistic budget in which each annotator can label only 50 items.

We compare three ways to spend the same $3 \times 50 = 150$ annotations: \textbf{depth-first}, where all annotators label items 1--50 in data order; \textbf{random spread}, where each annotator independently samples 50 items; and \textbf{stratified shared overlap}, where all annotators label the same 50 items chosen by stratified sampling from the judge labels.
For random spread and stratified shared overlap, we run 500 Monte Carlo trials across all 4 judges and compare each sparse decision to the full-data ground truth.

Table~\ref{tab:app_depth_breadth} reports the exact outcome.
Depth-first is badly biased because the early slice overrepresents higher-severity cases, inflating the human ceiling.
Random spread is roughly unbiased on average but leaves only about 14 multi-annotated items, so the estimate is too noisy to support stable deployment decisions.
Stratified shared overlap keeps the overlap dense and representative, which is why it serves as the paper's motivating design recommendation.

\begin{table}[!htbp]
\centering
\scriptsize
\caption{Exact numbers for the motivating depth-vs.-breadth example on ISIC Lesion. Agreement is measured by Krippendorff's $\alpha$; the dense reference is $\hat\alpha_{\text{dense}} = 0.351$.}
\label{tab:app_depth_breadth}
\begin{tabular}{lcccc}
\toprule
Strategy & Items with $\geq 2$ human labels & Human agreement ($\hat\alpha$) & Bias vs.\ full & Wrong decisions \\
\midrule
Depth-first & 50 & 0.535 & $+$0.184 & 75.0\% \\
Random spread & $\approx 14$ & $0.340 \pm 0.183$ & $-$0.011 & 39.9\% \\
Stratified shared overlap & 50 & $0.353 \pm 0.062$ & $+$0.001 & 26.5\% \\
\bottomrule
\end{tabular}
\end{table}

\subsection{Synthetic Data Generation}
\label{app:synthetic}

\paragraph{Noise channel model.}
Let $y_i \sim \text{Cat}(\pi)$ denote the true label for item $i$, drawn i.i.d.\ from a categorical distribution with prevalence vector $\pi = (\pi_1, \ldots, \pi_L)$.
Each annotator $k$ (human or judge) independently produces
\[
A_{ik} = \begin{cases}
y_i & \text{with probability } p_k, \\
\text{Uniform}(\calC \setminus \{y_i\}) & \text{with probability } 1 - p_k,
\end{cases}
\]
where $p_k$ is the annotator's agreement rate: $p_h$ for humans, $p_j$ for the judge.
All humans share the same $p_h$, making them exchangeable; the judge's $p_j$ is varied across experiments.

\paragraph{Default parameters.}
Table~\ref{tab:app_synth_params} summarises the two parameter profiles used across the paper.
The \emph{main} profile ($p_h = 0.85$, 300 trials) applies to all experiments in \S\S\ref{sec:rq1}--\ref{sec:rq3} and their appendix tables unless otherwise noted.
The \emph{sensitivity} profile ($p_h = 0.80$, varying $n$, $K$, and trial counts) applies to the ordinal-label, corpus-size, budget-tradeoff, and correlated-item analyses in \S\ref{sec:rq4} and Appendix~\ref{app:rq4_full}; each caption states its exact configuration.

\begin{table}[!htbp]
\centering
\scriptsize
\caption{Default synthetic-experiment parameter profiles. Every caption states the exact settings used; this table serves as a quick reference.}
\label{tab:app_synth_params}
\begin{tabular}{lcc}
\toprule
Parameter & Main profile & Sensitivity profile \\
\midrule
Corpus size $n$ & 500 & 100--2\,000 \\
Human annotators $K$ & 4 & 2--10 \\
Label categories $L$ & 2 & 2 or 5 \\
Human accuracy $p_h$ & 0.85 & 0.80 \\
Tolerance $\varepsilon$ & 0.05 & 0.05 \\
Monte Carlo trials & 300 & 50--500 \\
Scoring function $F$ & $\hat p_o$ & $\hat p_o$ \\
\bottomrule
\end{tabular}
\end{table}

\paragraph{Model limitations.}
The symmetric noise channel does not capture correlated annotation errors, item-difficulty heterogeneity, or rater-specific biases.
Table~\ref{tab:app_correlated} in \S\ref{app:correlated} probes robustness to the first two: overlap guidelines hold under mild and moderate item clustering, but strong clustering (${\leq}10$ clusters) degrades reliability at low overlap, requiring $\rho \geq 0.25$ for recovery.

\section{Supporting Experiments}
\label{app:tables}

This section collects the full tables backing the main-text experiments. Main-text experiments use $p_h = 0.85$ and 300 trials (see \S\ref{sec:setup}). Several sensitivity analyses below (ordinal labels, correlated items, corpus size, budget tradeoff) use $p_h = 0.80$ and varying trial counts to probe a broader parameter space; qualitative findings are consistent across both settings and each caption states the exact configuration.

\subsection{Coefficient and Label-Count Details}
\label{app:rq1_full}
\label{app:decision_full}
\label{app:metric_comparison}

This subsection collects two related analyses backing \S\ref{sec:rq1}: coefficient-swap pass rates and ordinal-label ($L=5$) wrong-decision rates.

\paragraph{Coefficient swap.}
Table~\ref{tab:app_coef_swap} reports the deploy/reject pass rate when the base coefficient $F$ is swapped between $\hat p_o$, $\hat\alpha$, $\hat\kappa$, and $\widehat{\mathrm{AC1}}$. Under uniform prevalence, all coefficients perform near-perfectly on the strongest judges ($p_j = 0.95$); on the borderline gap ($p_j = 0.90$), $\hat{p}_o$ leads at $\rho = 0.05$ (94.7\% vs.\ 87--88\% for the others). Under skewed prevalence ($\pi_{\max} = 0.90$), $\hat{p}_o$ leads more decisively: at $\rho = 0.05$ with $p_j = 0.90$, $\hat{p}_o$ reaches 95.3\% while $\hat\alpha$ drops to 87.3\% and $\hat\kappa$ to 87.8\%. $\widehat{\mathrm{AC1}}$ remains the best chance-corrected alternative (93.2\%). This is the bias/variance flow-through of Theorem~\ref{thm:sparsity}: the diverging amplification $\gamma_\kappa$ (part~iii) and $\gamma_\alpha$ (part~ii) inject additional noise under skew.

\begin{table}[!htbp]
\centering
\scriptsize
\caption{Pass rates (\%) for $\omega \geq 0.5$ when the base coefficient $F$ is swapped. Synthetic $L = 2$, $n = 500$, $K = 4$, $p_h = 0.85$, $\varepsilon = 0.05$, averaged over random and stratified designs, 300 trials per condition.}
\label{tab:app_coef_swap}
\begin{tabular}{llcrrrr}
\toprule
& & & \multicolumn{4}{c}{$F$ (pass rate, \%)} \\
\cmidrule(lr){4-7}
Prevalence & $p_j$ (gt) & $\rho$ & $\hat{p}_o$ & $\hat{\alpha}$ & $\hat{\kappa}$ & $\widehat{\mathrm{AC1}}$ \\
\midrule
\multirow{6}{*}{Uniform} & \multirow{3}{*}{$0.90$ (pass)} & $0.05$ & 94.7 & 87.7 & 87.5 & 87.8 \\
 & & $0.10$ & 98.7 & 94.8 & 94.8 & 94.7 \\
 & & $0.25$ & 100.0 & 98.7 & 98.7 & 99.0 \\
\cmidrule(lr){2-7}
 & \multirow{3}{*}{$0.95$ (pass)} & $0.05$ & 99.8 & 98.7 & 98.7 & 98.7 \\
 & & $0.10$ & 99.8 & 99.3 & 99.3 & 99.3 \\
 & & $0.25$ & 100.0 & 100.0 & 100.0 & 100.0 \\
\midrule
\multirow{6}{*}{Skewed (.90)} & \multirow{3}{*}{$0.90$ (pass)} & $0.05$ & 95.3 & 87.3 & 87.8 & 93.2 \\
 & & $0.10$ & 98.8 & 90.7 & 91.5 & 97.5 \\
 & & $0.25$ & 99.8 & 97.3 & 97.3 & 99.7 \\
\cmidrule(lr){2-7}
 & \multirow{3}{*}{$0.95$ (pass)} & $0.05$ & 99.3 & 93.2 & 93.3 & 98.3 \\
 & & $0.10$ & 99.8 & 97.8 & 98.3 & 99.3 \\
 & & $0.25$ & 100.0 & 99.8 & 99.8 & 100.0 \\
\bottomrule
\end{tabular}
\end{table}

\paragraph{Ordinal labels ($L = 5$).}
\label{app:ordinal}
Table~\ref{tab:app_ordinal} repeats the wrong-decision analysis with $L = 5$ ordinal categories under uniform prevalence, comparing Krippendorff's $\alpha$ (nominal) against linearly-weighted $\kappa$. The results are consistent with the binary case: at $\rho = 0.05$, wrong decision rates range from 9--44\%, closely matching the $L = 2$ numbers. Weighted $\kappa$ shows slightly higher error rates than $\alpha$ (e.g., 24\% vs.\ 21\% for $p_j = 0.85$), confirming that overlap guidelines generalise beyond binary labels.

\begin{table}[!htbp]
\centering
\scriptsize
\caption{Wrong decision rates (\%) with ordinal labels ($L = 5$, uniform prevalence, $n = 500$, $K = 4$, $p_h = 0.80$, 50 datasets $\times$ 50 trials).}
\label{tab:app_ordinal}
\begin{tabular}{llrrrr}
\toprule
Metric & $p_j$ & $\rho = 0.05$ & $0.10$ & $0.25$ & $0.50$ \\
\midrule
\multirow{5}{*}{$\alpha$ (nominal)} & 0.70 & 37.3 & 29.2 & 19.9 & 11.6 \\
& 0.76 & 43.9 & 41.0 & 34.0 & 27.5 \\
& 0.80 & 33.9 & 29.3 & 19.6 & 9.0 \\
& 0.85 & 20.5 & 10.6 & 2.8 & 0.5 \\
& 0.90 & 8.8 & 2.8 & 0.0 & 0.0 \\
\midrule
\multirow{5}{*}{Weighted $\kappa$} & 0.70 & 38.7 & 31.0 & 22.7 & 13.1 \\
& 0.76 & 41.5 & 40.5 & 32.3 & 24.8 \\
& 0.80 & 35.2 & 34.2 & 23.2 & 13.3 \\
& 0.85 & 24.1 & 13.5 & 4.8 & 0.8 \\
& 0.90 & 13.0 & 5.2 & 0.5 & 0.0 \\
\bottomrule
\end{tabular}
\end{table}

\subsection{Stratification Family on Real Benchmarks}
\label{app:rq2_full}
\label{app:skew_masked_full}

This subsection reports the full deploy/reject pivots for the stratification family comparison whose test-independent design-quality metrics were summarised in \S\ref{sec:rq2} (Table~\ref{tab:rq2_real}). All four family members are evaluated across two prevalence regimes: unmasked (natural prevalence) and skew-masked (binarised). Each condition uses 300 trials with the leave-one-out pipeline (Eqs.~\eqref{eq:per_rater}--\eqref{eq:winning_rate}).

\paragraph{Unmasked real benchmarks.} Table~\ref{tab:app_unmasked} shows the family on the original CeBaB-stars (5-class) and SummEval (5-point Likert) labels at their natural prevalence ($\pi_{\max} \in [0.26, 0.56]$). At low overlap ($\rho \in \{0.05, 0.10\}$), \textsc{Strat} matches or beats every competitor on both benchmarks. At $\rho = 0.25$ the gap among all designs collapses to within ${\sim}0.04$; the sequential members do not gain ground in either regime.

\begin{table}[!htbp]
\centering
\caption{Unmasked real-benchmark family pivot (mean wrong-decision rate across reject-judges, lower is better). ``$-$'' marks entries where \textsc{Seq-Coverage} could not be evaluated on SummEval because its reduced-coverage trials degenerated under SummEval's three-annotator setting. Bold marks the per-row winner.}
\label{tab:app_unmasked}
\small
\setlength{\tabcolsep}{4pt}
\begin{tabular}{llrrrrr}
\toprule
Dataset & $\pi_{\max}$ & $\rho$ & \textsc{Random} & \textsc{Strat} & \textsc{Seq-Ref.} & \textsc{Seq-Cov.} \\
\midrule
\multirow{3}{*}{cebab\_stars} & \multirow{3}{*}{0.26} & .05 & .727 & \textbf{.552} & .680 & .648 \\
 & & .10 & .513 & \textbf{.442} & .557 & .470 \\
 & & .25 & \textbf{.240} & .258 & .270 & .275 \\
\midrule
\multirow{3}{*}{summeval} & \multirow{3}{*}{0.55} & .05 & .163 & \textbf{.000} & .168 & $-$ \\
 & & .10 & .024 & \textbf{.000} & .028 & $-$ \\
 & & .25 & \textbf{.000} & \textbf{.000} & \textbf{.000} & $-$ \\
\bottomrule
\end{tabular}
\end{table}

\paragraph{Skew-masked real benchmarks.} To probe the high-skew regime, we binarize each benchmark at one or more thresholds $t$, producing six configurations spanning $\pi_{\max} \in [0.55, 0.96]$.

\begin{table}[!htbp]
\centering
\caption{Skew-masked real-benchmark family pivot (300 trials per condition, mean wrong-decision rate across reject judges). Bold marks the per-row winner.}
\label{tab:app_skew_masked_full}
\small
\setlength{\tabcolsep}{4pt}
\begin{tabular}{llrrrrr}
\toprule
Dataset & $\pi_{\max}$ & $\rho$ & \textsc{Random} & \textsc{Strat} & \textsc{Seq-Ref.} & \textsc{Seq-Cov.} \\
\midrule
\multirow{3}{*}{summeval@$t{=}3$} & \multirow{3}{*}{0.55} & .05 & .000 & .000 & .000 & .000 \\
 & & .10 & .000 & .000 & .000 & .000 \\
 & & .25 & .000 & .000 & .000 & .000 \\
\midrule
\multirow{3}{*}{lesion@$t{=}0$} & \multirow{3}{*}{0.65} & .05 & .110 & .075 & \textbf{.045} & .075 \\
 & & .10 & .010 & \textbf{.005} & \textbf{.005} & .010 \\
 & & .25 & \textbf{.000} & \textbf{.000} & \textbf{.000} & \textbf{.000} \\
\midrule
\multirow{3}{*}{summeval@$t{=}2$} & \multirow{3}{*}{0.72} & .05 & .001 & .002 & \textbf{.000} & .002 \\
 & & .10 & .000 & .000 & .000 & .000 \\
 & & .25 & .000 & .000 & .000 & .000 \\
\midrule
\multirow{3}{*}{lesion@$t{=}1$} & \multirow{3}{*}{0.74} & .05 & .310 & .306 & .294 & \textbf{.276} \\
 & & .10 & .159 & .248 & .169 & \textbf{.128} \\
 & & .25 & .084 & .149 & \textbf{.060} & .064 \\
\midrule
\multirow{3}{*}{lesion@$t{=}2$} & \multirow{3}{*}{0.93} & .05 & .557 & \textbf{.487} & .545 & .567 \\
& & .10 & \textbf{.357} & .365 & .377 & .382 \\
 & & .25 & .168 & .193 & .143 & \textbf{.098} \\
\midrule
\multirow{3}{*}{cebab\_aspects@$t{=}0$} & \multirow{3}{*}{0.96} & .05 & .886 & \textbf{.406} & .882 & .956 \\
 & & .10 & .513 & \textbf{.149} & .488 & .815 \\
 & & .25 & .042 & \textbf{.013} & .032 & .059 \\
\bottomrule
\end{tabular}
\end{table}

\paragraph{Summary at high skew.}
At $\pi_{\max} \geq 0.90$, out of 54 paired comparisons, \textsc{Strat} beats the sequential alternative by $\geq 0.10$ in 28 (52\%); the sequential design beats \textsc{Strat} by $\geq 0.10$ in only 1 (2\%); and the mean improvement of moving from \textsc{Strat} to a sequential design is $-0.234$ (i.e.\ \textsc{Strat} is on average 23\,pp better at high skew). These numbers support the recommendation in \S\ref{sec:rq2} that static stratification is the preferred default.

\subsection{Real LLM Judge Details}
\label{app:rq3_full}
\label{app:ranking_stability}

\paragraph{Ground-truth rankings (by $\hat p_o$).}
Table~\ref{tab:app_gt_rankings} lists the dense-reference $\hat p_o$ scores that define each judge's ground-truth ranking and deploy/reject status on each benchmark.
Judges are sorted by mean $\hat p_o$ across benchmarks.
Per-benchmark rankings vary substantially: for instance, Gemini Pro ranks 1st on WAX but 10th on SummEval, illustrating why ranking stability requires high overlap (\S\ref{sec:rq3}).

\begin{table}[!htbp]
\centering
\scriptsize
\caption{Ground-truth $\hat p_o$ scores (dense reference, $\rho = 1$) for each judge--benchmark pair. Judges sorted by mean score across the four benchmarks.}
\label{tab:app_gt_rankings}
\setlength{\tabcolsep}{4pt}
\begin{tabular}{lrrrr}
\toprule
Judge & WAX & CeBaB-asp. & CeBaB-stars & SummEval \\
\midrule
GPT-5.4       & .293 & .899 & .635 & .386 \\
Gemini Pro    & .374 & .877 & .585 & .201 \\
GPT-5.2       & .305 & .860 & .675 & .378 \\
Cl.\ Sonnet 4.5 & .325 & .868 & .585 & .342 \\
Cl.\ Opus 4.5  & .350 & .865 & .531 & .349 \\
GPT-4o        & .264 & .860 & .655 & .263 \\
Gemini Flash  & .329 & .864 & .485 & .226 \\
Llama-3.1     & .203 & .816 & .570 & .350 \\
GPT-4o-mini   & .220 & .833 & .590 & .276 \\
Mistral-v0.3  & .154 & .763 & .460 & .456 \\
\bottomrule
\end{tabular}
\end{table}

\subsubsection{Per-Judge Wrong-Decision Rates}
\label{app:real_strat}

Tables~\ref{tab:rq3_real} and~\ref{tab:rq3_borderline} in the main body summarise aggregated wrong-decision rates. This section provides per-judge detail for each benchmark under the leave-one-out pipeline ($F = \hat p_o$, $\varepsilon = 0.05$, 300 trials per condition).

\paragraph{SummEval.}
On SummEval (3 human annotators, 6\,400 dimension-level items), all 10 judges have dense $\omega = 0$ and are correctly rejected at every overlap rate and design (0\% wrong decisions uniformly); a per-judge table is omitted because every entry is zero.

\paragraph{CeBaB (star ratings).}
Table~\ref{tab:app_real_strat_cebab} shows per-judge wrong-decision rates at $\rho \in \{0.05, 0.10, 0.25, 0.50\}$. Nine of ten judges are ground-truth pass; only Gemini Flash ($\omega = 0.20$) is reject. Mistral-v0.3 ($\omega = 0.50$) sits exactly on the boundary: its wrong-decision rate barely decays with $\rho$.

\begin{table}[!htbp]
\centering
\caption{Wrong-decision rates on CeBaB star ratings ($F = \hat p_o$, $\varepsilon = 0.05$, 300 trials). Bold = best design per row.}
\label{tab:app_real_strat_cebab}
\small
\setlength{\tabcolsep}{3.5pt}
\begin{tabular}{llrrrrr}
\toprule
Judge (decision) & $\rho$ & \textsc{Rand.} & \textsc{Strat} & \textsc{Seq-R} & \textsc{Seq-C} & Type \\
\midrule
\multirow{4}{*}{\shortstack[l]{GPT-5.2\\(pass, $\omega{=}1.00$)}}
  & .05 & .130 & \textbf{.037} & .103 & .217 & FR \\
  & .10 & .037 & \textbf{.017} & .037 & .030 & FR \\
  & .25 & .000 & .000 & .000 & .000 & FR \\
  & .50 & .000 & .000 & .000 & .000 & FR \\
\midrule
\multirow{4}{*}{\shortstack[l]{GPT-4o\\(pass, $\omega{=}1.00$)}}
  & .05 & .113 & \textbf{.020} & .180 & .223 & FR \\
  & .10 & .040 & \textbf{.007} & .027 & .060 & FR \\
  & .25 & .000 & .000 & .003 & .007 & FR \\
  & .50 & .000 & .000 & .000 & .000 & FR \\
\midrule
\multirow{4}{*}{\shortstack[l]{GPT-5.4\\(pass, $\omega{=}1.00$)}}
  & .05 & .153 & \textbf{.053} & .163 & .263 & FR \\
  & .10 & .047 & \textbf{.017} & .043 & .113 & FR \\
  & .25 & .003 & .000 & .000 & .007 & FR \\
  & .50 & .000 & .000 & .000 & .000 & FR \\
\midrule
\multirow{4}{*}{\shortstack[l]{GPT-4o-mini\\(pass, $\omega{=}1.00$)}}
  & .05 & .233 & \textbf{.067} & .227 & .353 & FR \\
  & .10 & .080 & \textbf{.017} & .093 & .093 & FR \\
  & .25 & .030 & .000 & .013 & .027 & FR \\
  & .50 & .000 & .000 & .000 & .000 & FR \\
\midrule
\multirow{4}{*}{\shortstack[l]{Llama-3.1\\(pass, $\omega{=}.90$)}}
  & .05 & .270 & \textbf{.113} & .297 & .430 & FR \\
  & .10 & .150 & \textbf{.077} & .163 & .260 & FR \\
  & .25 & .090 & \textbf{.010} & .083 & .077 & FR \\
  & .50 & .007 & .000 & .007 & .013 & FR \\
\midrule
\multirow{4}{*}{\shortstack[l]{Gemini Pro\\(pass, $\omega{=}.90$)}}
  & .05 & .247 & \textbf{.143} & .293 & .337 & FR \\
  & .10 & .170 & \textbf{.077} & .193 & .220 & FR \\
  & .25 & .087 & \textbf{.010} & .043 & .090 & FR \\
  & .50 & .000 & .003 & .003 & .000 & FR \\
\midrule
\multirow{4}{*}{\shortstack[l]{Cl.\ Sonnet 4.5\\(pass, $\omega{=}.90$)}}
  & .05 & .287 & \textbf{.140} & .303 & .413 & FR \\
  & .10 & .170 & \textbf{.070} & .160 & .217 & FR \\
  & .25 & .110 & \textbf{.010} & .100 & .067 & FR \\
  & .50 & .007 & .000 & .020 & .020 & FR \\
\midrule
\multirow{4}{*}{\shortstack[l]{Cl.\ Opus 4.5\\(pass, $\omega{=}.80$)}}
  & .05 & .423 & \textbf{.177} & .380 & .570 & FR \\
  & .10 & .280 & \textbf{.120} & .280 & .347 & FR \\
  & .25 & .217 & \textbf{.033} & .167 & .247 & FR \\
  & .50 & .020 & \textbf{.007} & .050 & .053 & FR \\
\midrule
\multirow{4}{*}{\shortstack[l]{Mistral-v0.3\\(pass, $\omega{=}.50$)}}
  & .05 & .503 & \textbf{.330} & .593 & .700 & FR \\
  & .10 & .503 & \textbf{.323} & .503 & .617 & FR \\
  & .25 & .607 & \textbf{.300} & .597 & .540 & FR \\
  & .50 & .497 & \textbf{.267} & .503 & .483 & FR \\
\midrule
\multirow{4}{*}{\shortstack[l]{Gemini Flash\\(reject, $\omega{=}.20$)}}
  & .05 & .433 & .567 & .453 & \textbf{.307} & FA \\
  & .10 & .423 & .520 & .480 & \textbf{.383} & FA \\
  & .25 & .307 & .320 & \textbf{.250} & .357 & FA \\
  & .50 & .210 & .207 & .223 & \textbf{.307} & FA \\
\bottomrule
\end{tabular}
\end{table}

\paragraph{WAX.}
Table~\ref{tab:app_real_strat_wax} shows WAX results (246 items, 8 humans). Seven judges are ground-truth pass, three are reject.

\begin{table}[!htbp]
\centering
\caption{Wrong-decision rates on WAX ($F = \hat p_o$, $\varepsilon = 0.05$, 300 trials). Bold = best design per row.}
\label{tab:app_real_strat_wax}
\small
\setlength{\tabcolsep}{3.5pt}
\begin{tabular}{llrrrrr}
\toprule
Judge (decision) & $\rho$ & \textsc{Rand.} & \textsc{Strat} & \textsc{Seq-R} & \textsc{Seq-C} & Type \\
\midrule
\multirow{4}{*}{\shortstack[l]{Gemini Pro\\(pass, $\omega{=}.75$)}}
  & .05 & .370 & \textbf{.033} & .360 & .433 & FR \\
  & .10 & .303 & \textbf{.023} & .343 & .493 & FR \\
  & .25 & .033 & \textbf{.000} & .037 & .087 & FR \\
  & .50 & .000 & .000 & .000 & .000 & FR \\
\midrule
\multirow{4}{*}{\shortstack[l]{Gemini Flash\\(pass, $\omega{=}.75$)}}
  & .05 & .553 & \textbf{.180} & .643 & .623 & FR \\
  & .10 & .510 & \textbf{.073} & .503 & .650 & FR \\
  & .25 & .217 & \textbf{.017} & .200 & .283 & FR \\
  & .50 & .020 & .000 & \textbf{.010} & .020 & FR \\
\midrule
\multirow{4}{*}{\shortstack[l]{GPT-5.2\\(pass, $\omega{=}.625$)}}
  & .05 & .637 & \textbf{.473} & .637 & .670 & FR \\
  & .10 & .690 & \textbf{.360} & .610 & .683 & FR \\
  & .25 & .527 & \textbf{.147} & .530 & .657 & FR \\
  & .50 & .180 & \textbf{.080} & .213 & .230 & FR \\
\midrule
\multirow{4}{*}{\shortstack[l]{Cl.\ Opus 4.5\\(pass, $\omega{=}.625$)}}
  & .05 & .540 & \textbf{.290} & .583 & .520 & FR \\
  & .10 & .500 & \textbf{.150} & .420 & .537 & FR \\
  & .25 & .257 & \textbf{.050} & .240 & .307 & FR \\
  & .50 & .023 & \textbf{.000} & .050 & .047 & FR \\
\midrule
\multirow{4}{*}{\shortstack[l]{Cl.\ Sonnet 4.5\\(pass, $\omega{=}.625$)}}
  & .05 & .583 & \textbf{.313} & .537 & .573 & FR \\
  & .10 & .630 & \textbf{.177} & .553 & .693 & FR \\
  & .25 & .320 & \textbf{.073} & .383 & .447 & FR \\
  & .50 & .043 & \textbf{.003} & .067 & .057 & FR \\
\midrule
\multirow{4}{*}{\shortstack[l]{GPT-4o\\(pass, $\omega{=}.50$)}}
  & .05 & .783 & \textbf{.543} & .743 & .783 & FR \\
  & .10 & .863 & \textbf{.430} & .770 & .843 & FR \\
  & .25 & .777 & \textbf{.397} & .817 & .833 & FR \\
  & .50 & .600 & \textbf{.450} & .663 & .690 & FR \\
\midrule
\multirow{4}{*}{\shortstack[l]{GPT-5.4\\(pass, $\omega{=}.50$)}}
  & .05 & .733 & \textbf{.507} & .570 & .597 & FR \\
  & .10 & .670 & \textbf{.443} & .703 & .813 & FR \\
  & .25 & .643 & \textbf{.330} & .610 & .720 & FR \\
  & .50 & .470 & \textbf{.240} & .467 & .497 & FR \\
\midrule
\multirow{4}{*}{\shortstack[l]{Llama-3.1\\(reject, $\omega{=}.375$)}}
  & .05 & .127 & .457 & .100 & \textbf{.087} & FA \\
  & .10 & .053 & .307 & .040 & \textbf{.030} & FA \\
  & .25 & .050 & .297 & .043 & \textbf{.017} & FA \\
  & .50 & .093 & .197 & \textbf{.027} & \textbf{.020} & FA \\
\midrule
\multirow{4}{*}{\shortstack[l]{GPT-4o-mini\\(reject, $\omega{=}.375$)}}
  & .05 & .157 & .517 & .143 & \textbf{.123} & FA \\
  & .10 & .060 & .443 & .083 & \textbf{.040} & FA \\
  & .25 & .093 & .487 & .117 & \textbf{.087} & FA \\
  & .50 & .207 & .623 & .223 & \textbf{.130} & FA \\
\midrule
\multirow{4}{*}{\shortstack[l]{Mistral-v0.3\\(reject, $\omega{=}0$)}}
  & .05 & .070 & .330 & \textbf{.020} & \textbf{.013} & FA \\
  & .10 & .027 & .127 & \textbf{.007} & \textbf{.000} & FA \\
  & .25 & .007 & .027 & \textbf{.000} & .003 & FA \\
  & .50 & .000 & .000 & .000 & .000 & FA \\
\bottomrule
\end{tabular}
\end{table}

\paragraph{CeBaB (aspects).}
Table~\ref{tab:app_real_strat_cebab_asp} shows CeBaB-aspects results (1\,008 aspect-level items, 10 humans). Eight judges are ground-truth pass, two are reject.

\begin{table}[!htbp]
\centering
\caption{Wrong-decision rates on CeBaB aspects ($F = \hat p_o$, $\varepsilon = 0.05$, 300 trials). Bold = best design per row.}
\label{tab:app_real_strat_cebab_asp}
\small
\setlength{\tabcolsep}{3.5pt}
\begin{tabular}{llrrrrr}
\toprule
Judge (decision) & $\rho$ & \textsc{Rand.} & \textsc{Strat} & \textsc{Seq-R} & \textsc{Seq-C} & Type \\
\midrule
\multirow{4}{*}{\shortstack[l]{GPT-5.4\\(pass, $\omega{=}.90$)}}
  & .05 & .097 & \textbf{.090} & .147 & .293 & FR \\
  & .10 & .113 & \textbf{.040} & .073 & .160 & FR \\
  & .25 & .017 & \textbf{.010} & .033 & .017 & FR \\
  & .50 & .000 & .000 & .000 & .003 & FR \\
\midrule
\multirow{4}{*}{\shortstack[l]{Cl.\ Sonnet 4.5\\(pass, $\omega{=}.90$)}}
  & .05 & .163 & \textbf{.090} & .123 & .380 & FR \\
  & .10 & .107 & \textbf{.047} & .143 & .317 & FR \\
  & .25 & .030 & \textbf{.003} & .030 & .067 & FR \\
  & .50 & .000 & .000 & .000 & .000 & FR \\
\midrule
\multirow{4}{*}{\shortstack[l]{Gemini Pro\\(pass, $\omega{=}.80$)}}
  & .05 & .147 & \textbf{.073} & .227 & .437 & FR \\
  & .10 & .120 & \textbf{.033} & .153 & .297 & FR \\
  & .25 & .043 & \textbf{.000} & .023 & .057 & FR \\
  & .50 & .003 & .000 & .000 & .000 & FR \\
\midrule
\multirow{4}{*}{\shortstack[l]{GPT-4o\\(pass, $\omega{=}.80$)}}
  & .05 & .203 & \textbf{.067} & .223 & .477 & FR \\
  & .10 & .140 & \textbf{.033} & .127 & .293 & FR \\
  & .25 & .053 & \textbf{.000} & .037 & .047 & FR \\
  & .50 & .000 & .000 & .000 & .000 & FR \\
\midrule
\multirow{4}{*}{\shortstack[l]{GPT-5.2\\(pass, $\omega{=}.80$)}}
  & .05 & .230 & \textbf{.157} & .287 & .487 & FR \\
  & .10 & .233 & \textbf{.117} & .207 & .320 & FR \\
  & .25 & .097 & \textbf{.043} & .117 & .167 & FR \\
  & .50 & .027 & \textbf{.007} & .020 & .030 & FR \\
\midrule
\multirow{4}{*}{\shortstack[l]{Cl.\ Opus 4.5\\(pass, $\omega{=}.80$)}}
  & .05 & .213 & \textbf{.117} & .227 & .547 & FR \\
  & .10 & .210 & \textbf{.093} & .197 & .367 & FR \\
  & .25 & .130 & \textbf{.033} & .053 & .273 & FR \\
  & .50 & .023 & \textbf{.007} & .027 & .027 & FR \\
\midrule
\multirow{4}{*}{\shortstack[l]{Gemini Flash\\(pass, $\omega{=}.70$)}}
  & .05 & .217 & \textbf{.200} & .217 & .447 & FR \\
  & .10 & .233 & \textbf{.203} & .267 & .410 & FR \\
  & .25 & .187 & \textbf{.107} & .163 & .233 & FR \\
  & .50 & .063 & \textbf{.033} & .070 & .107 & FR \\
\midrule
\multirow{4}{*}{\shortstack[l]{Llama-3.1\\(pass, $\omega{=}.50$)}}
  & .05 & .367 & \textbf{.380} & .380 & .623 & FR \\
  & .10 & .353 & \textbf{.360} & .430 & .563 & FR \\
  & .25 & .370 & \textbf{.263} & .450 & .497 & FR \\
  & .50 & .290 & \textbf{.223} & .293 & .367 & FR \\
\midrule
\multirow{4}{*}{\shortstack[l]{GPT-4o-mini\\(reject, $\omega{=}.40$)}}
  & .05 & .643 & .727 & .620 & \textbf{.393} & FA \\
  & .10 & .620 & .730 & .657 & \textbf{.507} & FA \\
  & .25 & .667 & .697 & .693 & \textbf{.693} & FA \\
  & .50 & .760 & .740 & .760 & \textbf{.750} & FA \\
\midrule
\multirow{4}{*}{\shortstack[l]{Mistral-v0.3\\(reject, $\omega{=}.20$)}}
  & .05 & .363 & .150 & .230 & \textbf{.127} & FA \\
  & .10 & .127 & .043 & .110 & \textbf{.187} & FA \\
  & .25 & .010 & .003 & .013 & \textbf{.040} & FA \\
  & .50 & .000 & .000 & .000 & .000 & FA \\
\bottomrule
\end{tabular}
\end{table}

\subsection{Sensitivity to Setup Parameters}
\label{app:rq4_full}
\label{app:overlap_skew}

This subsection collects sensitivity analyses backing \S\ref{sec:rq4}.

\paragraph{Overlap vs.\ rater count.}
\label{app:k_effect}
Figure~\ref{fig:k_effect} shows diminishing returns from adding raters relative to increasing overlap. Curves flatten after $K = 5$ while each $\rho$ step yields a larger vertical jump under both uniform and skewed prevalence.

\begin{figure}[!htbp]
\centering
\includegraphics[width=0.55\columnwidth]{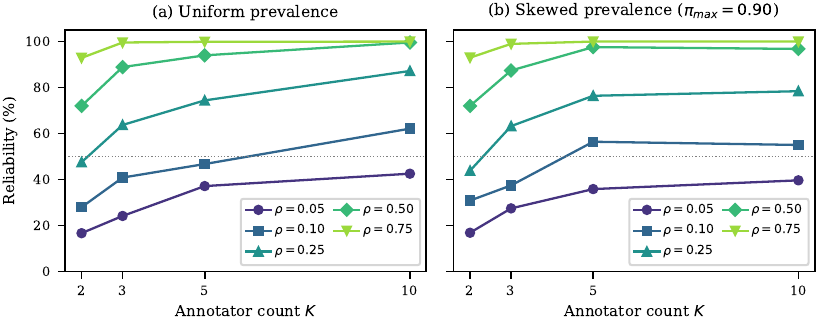}
\caption{\textbf{Reliability vs.\ annotator count at five overlap rates} ($n = 500$, $p_h = 0.80$). Reliability = \% of trials within $\pm 0.05$ of ground truth. Grey line marks 50\%.}
\label{fig:k_effect}
\end{figure}

\paragraph{Extreme-skew overlap requirements.}
Table~\ref{tab:overlap_required} in the main paper gives the uniform and moderate-skew overlap requirements. Table~\ref{tab:app_overlap_skew} extends this to extreme prevalence skew ($\pi_{\max} = 0.95$), providing the concrete backing for the paper's most conservative overlap recommendations.

\begin{table}[!htbp]
\centering
\scriptsize
\caption{Required overlap rate for extreme prevalence skew ($\pi_{\max} = 0.95$), $n = 1000$, $K = 5$. Ranges span target $F^{*} \in \{0.6, 0.7, 0.8, 0.9\}$.}
\label{tab:app_overlap_skew}
\begin{tabular}{llrrr}
\toprule
$L$ & Prevalence $\pi$ & $\delta\!=\!0.02$ & $\delta\!=\!0.05$ & $\delta\!=\!0.10$ \\
\midrule
2 & (.95, .05) & $\geq$50\% & 16--50\% & 5--22\% \\
5 & (.85, .05, .04, .03, .03) & $\geq$49\% & 13--42\% & 4--16\% \\
\bottomrule
\end{tabular}
\end{table}

\paragraph{Effect of corpus size $n$.}
\label{app:n_items}
Table~\ref{tab:app_n_items} shows that reliability primarily tracks the absolute overlap count $m = \rho n$, so larger corpora can compensate for low overlap rates by increasing the number of multi-annotated items.

\begin{table}[!htbp]
\centering
\scriptsize
\caption{Reliability (\% within $\pm 0.05$) as a function of corpus size $n$ and overlap rate $\rho$ ($K = 5$, $p_h = 0.80$, 500 trials).}
\label{tab:app_n_items}
\begin{tabular}{llrrrr}
\toprule
$n$ & Prevalence & $\rho = 0.05$ & $0.10$ & $0.25$ & $0.50$ \\
\midrule
\multirow{2}{*}{100} & Uniform & 20.6 & 28.8 & 47.8 & 73.4 \\
 & Skewed (.90) & 10.6 & 16.0 & 39.2 & 55.0 \\
\midrule
\multirow{2}{*}{200} & Uniform & 25.4 & 37.8 & 61.4 & 86.6 \\
 & Skewed (.90) & 26.2 & 34.8 & 52.6 & 84.6 \\
\midrule
\multirow{2}{*}{500} & Uniform & 36.4 & 55.4 & 81.2 & 98.2 \\
 & Skewed (.90) & 30.4 & 50.6 & 71.4 & 97.0 \\
\midrule
\multirow{2}{*}{1000} & Uniform & 54.4 & 70.0 & 96.0 & 100.0 \\
 & Skewed (.90) & 49.6 & 66.2 & 90.6 & 99.6 \\
\midrule
\multirow{2}{*}{2000} & Uniform & 72.0 & 89.8 & 99.2 & 100.0 \\
 & Skewed (.90) & 62.0 & 80.4 & 96.0 & 100.0 \\
\bottomrule
\end{tabular}
\end{table}

\paragraph{Depth-vs-breadth budget tradeoff.}
\label{app:cost_budget}
Under a fixed annotation budget $B$ (total annotation slots), the constraint $B \approx n(1 + \rho(K{-}1))$ trades item coverage $n$ against overlap depth $\rho$. Table~\ref{tab:app_cost_budget} shows that at every budget level, increasing overlap from $\rho = 0.05$ to $\rho = 0.25$ reduces wrong decisions by 6--13\,pp, with diminishing returns beyond $\rho \approx 0.30$.

\begin{table}[!htbp]
\centering
\scriptsize
\caption{Wrong decision rates (\%) under budget constraints ($K = 4$, $p_h = 0.80$, $p_j = 0.76$, uniform prevalence, 150 datasets).}
\label{tab:app_cost_budget}
\begin{tabular}{lrrrrr}
\toprule
& \multicolumn{5}{c}{Overlap rate $\rho$} \\
\cmidrule(lr){2-6}
Budget $B$ & 0.05 & 0.10 & 0.25 & 0.40 & 0.50 \\
\midrule
500 & 44.5 & 38.7 & 32.0 & 30.3 & 25.9 \\
1000 & 44.0 & 40.9 & 32.5 & 28.5 & 23.3 \\
2000 & 43.2 & 35.6 & 30.3 & 28.3 & 20.7 \\
5000 & 38.9 & 41.6 & 32.7 & 27.9 & 22.6 \\
\bottomrule
\end{tabular}
\end{table}

\section{Extended Comparison with Related Methods}
\label{app:prior_work_comparison}

Table~\ref{tab:app_prior_work} compares several related methods along five dimensions. The key distinction is that our \textsc{Strat} addresses \emph{prospective} allocation for \emph{agreement-based deployment decisions}, whereas all other methods are either post-hoc or target different estimands.

\begin{table}[!htbp]
\centering
\revised{
\caption{\textbf{Comparison with related methods.} ``Timing'' indicates whether the method operates before (prospective) or after (post-hoc) data collection. ``Estimand'' is what is being estimated or optimised.}
\label{tab:app_prior_work}
\scriptsize
\setlength{\tabcolsep}{4pt}
\begin{tabular}{lllll}
\toprule
Method & Estimand & Timing & Data structure & Overlap concept \\
\midrule
StratPPI \citep{fisch2024stratppi} & Population mean quality & Post-hoc & 1 human + 1 model pred.\ per item & Absent \\
PPI \citep{angelopoulos2023ppi} & Population statistic & Post-hoc & Same as StratPPI & Absent \\
SPA \citep{norregaard2022spa} & Agreement coeff.\ from sparse matrix & Post-hoc & Already-collected sparse matrix & Central but fixed \\
BIBD \citep{fleiss2003statistical} & Label allocation & Prospective & Symmetric rater pool & Symmetric blocks \\
Active allocation \citep{sheng2008get} & Redundancy for label quality & Post-hoc / online & Any; targets aggregation quality & Implicit \\
\midrule
\textbf{Our STRAT} & \textbf{Reliable deployment decision} & \textbf{Prospective} & \textbf{Asymmetric: judge + primary complete;} & \textbf{Central --- we design it} \\
& & & \textbf{sparse secondaries} & \\
\bottomrule
\end{tabular}
}
\end{table}

\revised{\textbf{StratPPI} \citep{fisch2024stratppi} applies stratification to bias-correct a population mean quality score using model predictions as control variates. There is no notion of inter-annotator agreement, annotation overlap, or judge accept/reject decisions. Stratification in StratPPI apportions bias-correction weights post-hoc; in our STRAT it determines \emph{which items} receive overlap annotations before data collection.}

\revised{\textbf{SPA} \citep{norregaard2022spa} is the most closely related prior work: it shows that agreement estimated from an already-collected sparse matrix is noisy, motivating the use of better estimators. Our work is complementary and extends SPA's insight prospectively: given that sparse estimation is unreliable, \emph{how should the annotation budget be allocated before data collection} to produce a reliable deployment decision? STRAT is the answer to this design question, which SPA does not address.}

\revised{\textbf{BIBD} \citep{fleiss2003statistical} provides balanced assignment of items to annotators for reliability studies. Our setting is inherently asymmetric (judge + primary annotator label all $n$ items; sparse secondaries overlap a subset), a structure that BIBD cannot represent. Furthermore, BIBD optimises design for label quality, not for deployment-decision reliability.}

\revised{These methods are complementary: one could use STRAT to select overlap items, SPA to estimate agreement from the resulting matrix, and StratPPI to tighten CIs on a downstream quality metric.}

\paragraph{Correlated items.}
\label{app:correlated}
The i.i.d.\ item assumption underlying Theorem~\ref{thm:sparsity} may not hold when items cluster by difficulty. Table~\ref{tab:app_correlated} tests this by generating items in clusters of shared difficulty ($n = 500$, $K = 5$, $L = 2$, $p_h = 0.85$, 300 trials). Under i.i.d., mild (50 clusters), and moderate (20 clusters) correlation, reliability stays above 96\% at all overlap rates. Under strong clustering (10 clusters of ${\sim}50$ items), reliability collapses to 10\% at $\rho = 0.05$ because a 25-item sample may capture only 1--3 clusters and severely misrepresent the difficulty distribution. Recovery requires $\rho \geq 0.25$ (72\%) or $\rho \geq 0.50$ (99\%). Practitioners should increase overlap or use cluster-aware stratification in such settings.

\begin{table}[!htbp]
\centering
\caption{Sensitivity to correlated items. Reliability (\% within $\pm 0.05$ of dense estimate) under varying item correlation ($n = 500$, $K = 5$, $L = 2$, $p_h = 0.85$, 300 trials).}
\label{tab:app_correlated}
\scriptsize
\begin{tabular}{lrrrr}
\toprule
Correlation & $\rho = 0.05$ & $0.10$ & $0.25$ & $0.50$ \\
\midrule
i.i.d.\ (500 clusters) & 99.0 & 97.0 & 97.7 & 100.0 \\
Mild (50 clusters) & 98.0 & 96.7 & 99.3 & 99.7 \\
Moderate (20 clusters) & 99.7 & 98.3 & 99.0 & 100.0 \\
Strong (10 clusters) & 10.0 & 31.3 & 72.3 & 99.0 \\
\bottomrule
\end{tabular}
\end{table}

\section{Practitioner Decision Guides}
\label{app:practitioner_guides}

\subsection{Decision Guide for Weak Stratification Signal}
\label{app:strat_guide}

\revised{When the primary annotator's label distribution may be uninformative for stratification, the following steps determine the appropriate fallback:
\begin{enumerate}[leftmargin=*, itemsep=2pt, topsep=2pt]
\item \textbf{Diagnose stratum entropy.} Compute $H = -\sum_\ell \hat\pi_\ell \log_2 \hat\pi_\ell$ from the primary annotator's labels. If $H < 0.5$ bits (primary annotator assigns $>90\%$ of items to one category, i.e.\ $\pi_{\max} > 0.90$), the stratification signal is weak and \textsc{Strat} may not reduce bias.
\item \textbf{Fall back to \textsc{Seq-Coverage}.} This design equalises pairwise coverage across raters without relying on representative strata. It outperforms \textsc{Strat} in the high-skew regime (Table~\ref{tab:rq2_real}, \texttt{lesion}@$t{=}2$ and \texttt{cebab\_aspects}@$t{=}0$ rows). The operational cost is the same as \textsc{Strat}: rater $k$'s task is fixed before their session using labels already in hand.
\item \textbf{Increase $\rho$.} Overlap quantity dominates design choice when stratification fails. At $\rho = 0.25$ under high skew, design differences among all four family members shrink to within ${\sim}0.05$ reliability (Table~\ref{tab:rq2_real}). If the budget permits only one lever, increasing $\rho$ is more reliable than switching designs.
\item \textbf{Richer stratification signal.} If richer item-level features are available (e.g.\ item difficulty estimates, auxiliary classifier confidence), these can replace the primary annotator's label as the stratum signal, restoring \textsc{Strat}'s bias-reduction advantage.
\end{enumerate}}

\subsection{Decision Protocol When False-Approval Cost Dominates}
\label{app:fa_guide}

\revised{\textsc{Strat} lowers false-rejection (FR) rate but raises false-approval (FA) rate relative to \textsc{Random} at low overlap (Table~\ref{tab:rq3_real}, $\rho=0.05$: FR 14.2\% vs 29.0\%; FA 17.2\% vs 11.2\%). The mechanism: \textsc{Random} can bias the estimated human ceiling upward, making the bar harder for weak judges to clear; \textsc{Strat} corrects this bias, which paradoxically lets some borderline-reject judges through. When false-approval cost far exceeds false-rejection cost, the following protocol applies:
\begin{enumerate}[leftmargin=*, itemsep=2pt, topsep=2pt]
\item \textbf{Increase $\rho$.} Overlap quantity is the dominant FA lever at all design levels. At $\rho = 0.25$, FA drops to 7--11\% regardless of design choice (Table~\ref{tab:rq3_real}); no design switch achieves this at $\rho = 0.05$.
\item \textbf{Raise the acceptance threshold.} Changing the accept rule from $\omega \geq 0.5$ to $\omega \geq 0.6$ reduces FA by raising the bar for approval. This is especially effective for borderline judges whose dense $\omega$ sits near 0.5.
\item \textbf{Use \textsc{Strat} with $\rho \geq 0.25$} for the lowest overall WDR (9.2\%, Table~\ref{tab:rq3_real}). \textsc{Strat} is not the lowest-FA design, but it provides the best trade-off between FR and FA at this overlap level.
\item \textbf{Use \textsc{Seq-Coverage} at low overlap} ($\rho \leq 0.10$) when FA is the primary concern. \textsc{Seq-Coverage} achieves the lowest FA across all low-overlap settings (6.6--7.2\% at $\rho = 0.05$--$0.10$), at the cost of higher FR. Note: at $\rho = 0.25$, \textsc{Seq-Refined} achieves slightly lower FA than \textsc{Seq-Coverage} (7.0\% vs.\ 7.5\%).
\end{enumerate}}

\clearpage
\section*{NeurIPS Paper Checklist}

\begin{enumerate}

\item {\bf Claims}
    \item[] Question: Do the main claims made in the abstract and introduction accurately reflect the paper's contributions and scope?
    \item[] Answer: \answerYes{}
    \item[] Justification: The abstract and introduction state the sparse-overlap problem, the theoretical and empirical contributions, and the scope of the conclusions; Sections~\ref{sec:problem}--\ref{sec:experiments} and Section~\ref{sec:conclusion} support those claims.

\item {\bf Limitations}
    \item[] Question: Does the paper discuss the limitations of the work performed by the authors?
    \item[] Answer: \answerYes{}
    \item[] Justification: Section~\ref{sec:conclusion} discusses simplifying assumptions, the empirical nature of the overlap guidelines, small-$K$ inference limits, and deployment caveats.

\item {\bf Theory assumptions and proofs}
    \item[] Question: For each theoretical result, does the paper provide the full set of assumptions and a complete (and correct) proof?
    \item[] Answer: \answerYes{}
    \item[] Justification: Section~\ref{sec:problem} states the assumptions for each theorem, and Appendix~\ref{app:proofs} provides proof sketches under those simplifying assumptions, but some derivation details are abbreviated.

    \item {\bf Experimental result reproducibility}
    \item[] Question: Does the paper fully disclose all the information needed to reproduce the main experimental results of the paper to the extent that it affects the main claims and/or conclusions of the paper (regardless of whether the code and data are provided or not)?
    \item[] Answer: \answerYes{}
    \item[] Justification: The paper provides substantial procedural detail for the synthetic study and reports the benchmark/model inventory, but it does not provide a full reproduction bundle for the proprietary judge runs with exact prompts, version identifiers, and run metadata.

\item {\bf Open access to data and code}
    \item[] Question: Does the paper provide open access to the data and code, with sufficient instructions to faithfully reproduce the main experimental results, as described in supplemental material?
    \item[] Answer: \answerYes{}
    \item[] Justification: We use public datasets and cite prior prompts/resources, but this submission does not provide an anonymized code release or a single bundled reproduction package.

\item {\bf Experimental setting/details}
    \item[] Question: Does the paper specify all the training and test details (e.g., data splits, hyperparameters, how they were chosen, type of optimizer) necessary to understand the results?
    \item[] Answer: \answerYes{}
    \item[] Justification: Section~\ref{sec:experiments} and Appendix~\ref{app:impl_data} describe the datasets, model inventory, synthetic settings, overlap regimes, and ground-truth construction, but the exact proprietary-model run settings are not fully specified.

\item {\bf Experiment statistical significance}
    \item[] Question: Does the paper report error bars suitably and correctly defined or other appropriate information about the statistical significance of the experiments?
    \item[] Answer: \answerYes{}
    \item[] Justification: The paper reports false-approval and wrong-decision rates with Monte Carlo standard errors aggregated over 300 trials per cell, with the pipeline definition in \S\ref{sec:pipeline} and per-judge per-design details in Appendix~\ref{app:real_strat}.

\item {\bf Experiments compute resources}
    \item[] Question: For each experiment, does the paper provide sufficient information on the computer resources (type of compute workers, memory, time of execution) needed to reproduce the experiments?
    \item[] Answer: \answerYes{}
    \item[] Justification: All synthetic and re-analysis experiments run on a single laptop CPU (Apple M-series, 16~GB RAM); each individual experiment script in \texttt{code/experiments/} completes in seconds-to-minutes, with the full reproduction sweep (Tables~\ref{tab:rq2_real}, \ref{tab:rq3_real}, \ref{tab:rq3_borderline}, all appendix tables, plus 300-trial Monte Carlo grids) finishing in well under one CPU-hour total. The five LLM judges newly collected for this work produced 28{,}822 labels via batched provider APIs (Azure OpenAI / Anthropic), which is the only step that requires external compute and depends on the provider's billed throughput rather than local resources.
    
\item {\bf Code of ethics}
    \item[] Question: Does the research conducted in the paper conform, in every respect, with the NeurIPS Code of Ethics \url{https://neurips.cc/public/EthicsGuidelines}?
    \item[] Answer: \answerYes{}
    \item[] Justification: The work studies evaluation reliability using previously collected benchmark annotations and existing LLM judge outputs, and we do not describe any deviations from the NeurIPS Code of Ethics.

\item {\bf Broader impacts}
    \item[] Question: Does the paper discuss both potential positive societal impacts and negative societal impacts of the work performed?
    \item[] Answer: \answerYes{}
    \item[] Justification: Section~\ref{sec:conclusion} discusses both the positive impact of more reliable judge validation and the negative impact of approving poor judges or rejecting strong ones under low overlap.
    
\item {\bf Safeguards}
    \item[] Question: Does the paper describe safeguards that have been put in place for responsible release of data or models that have a high risk for misuse (e.g., pre-trained language models, image generators, or scraped datasets)?
    \item[] Answer: \answerNA{}
    \item[] Justification: The paper does not release a new general-purpose model, scraped dataset, or other high-risk asset; it evaluates existing LLM judges and benchmark annotations.

\item {\bf Licenses for existing assets}
    \item[] Question: Are the creators or original owners of assets (e.g., code, data, models), used in the paper, properly credited and are the license and terms of use explicitly mentioned and properly respected?
    \item[] Answer: \answerNo{}
    \item[] Justification: The paper credits prior datasets and methods by citation, but the current draft does not enumerate licenses or terms of use for each external asset.

\item {\bf New assets}
    \item[] Question: Are new assets introduced in the paper well documented and is the documentation provided alongside the assets?
    \item[] Answer: \answerNA{}
    \item[] Justification: The submission does not introduce a documented new dataset, model, or software release as a research asset.

\item {\bf Crowdsourcing and research with human subjects}
    \item[] Question: For crowdsourcing experiments and research with human subjects, does the paper include the full text of instructions given to participants and screenshots, if applicable, as well as details about compensation (if any)? 
    \item[] Answer: \answerNA{}
    \item[] Justification: The experiments use existing benchmark annotations and LLM outputs; no new crowdsourcing or human-subject study is reported in this submission.

\item {\bf Institutional review board (IRB) approvals or equivalent for research with human subjects}
    \item[] Question: Does the paper describe potential risks incurred by study participants, whether such risks were disclosed to the subjects, and whether Institutional Review Board (IRB) approvals (or an equivalent approval/review based on the requirements of your country or institution) were obtained?
    \item[] Answer: \answerNA{}
    \item[] Justification: No new human-subject data collection or intervention is described in the paper.

\item {\bf Declaration of LLM usage}
    \item[] Question: Does the paper describe the usage of LLMs if it is an important, original, or non-standard component of the core methods in this research? Note that if the LLM is used only for writing, editing, or formatting purposes and does \emph{not} impact the core methodology, scientific rigor, or originality of the research, declaration is not required.
    \item[] Answer: \answerYes{}
    \item[] Justification: LLMs are central objects of study in this work, and Section~\ref{sec:experiments} explicitly lists the judge models, datasets, and analyses performed on them.

\end{enumerate}

\end{document}